\documentclass{article} 
\usepackage[preprint,nonatbib]{neurips_2020}
\usepackage{times}
\usepackage{subcaption}

\usepackage[utf8]{inputenc} 
\usepackage[T1]{fontenc}    

\usepackage{amsthm}
\usepackage{amsmath}
\usepackage{graphicx}
\usepackage[utf8]{inputenc} 
\usepackage[T1]{fontenc}    
\usepackage{hyperref}       
\usepackage{url}            
\usepackage{booktabs}       
\usepackage{amsfonts}       
\usepackage{nicefrac}       
\usepackage{microtype}      
\usepackage{xcolor}
\usepackage{graphicx}
\usepackage{algorithm}
\usepackage{algpseudocode}
\usepackage{subcaption}
\usepackage{multirow}
\usepackage{tabularx,colortbl,xcolor}

\newif\ifcomments
\commentsfalse

\ifcomments
\usepackage[textsize=small,color=lightgray]{todonotes}
\paperwidth=\dimexpr \paperwidth + 10cm\relax
\paperheight=\dimexpr \paperheight + 5cm\relax
\oddsidemargin=\dimexpr\oddsidemargin + 5cm\relax
\evensidemargin=\dimexpr\evensidemargin + 5cm\relax
\marginparwidth=\dimexpr \marginparwidth + 5cm\relax
\else
\usepackage[disable]{todonotes}
\fi

\newcommand{\yuandong}[1]{\todo[fancyline, color=red!50]{\textbf{Yuandong}: #1}\ignorespaces}

\def\imp{\emph{IMPs/b}}

\def\bnine{\texttt{baseline19}}
\def\bsix{\texttt{baseline16}}
\def\bs{\texttt{baseline}}
\def\vzero{\mathbf{0}}

\input{math_commands.sty}

\begin{document}


\title{Joint Policy Search for Multi-agent Collaboration with Imperfect Information}

\author{Yuandong Tian\thanks{\textbf{Author contribution:} Yuandong Tian proposed the theoretical framework, algorithm, proofs, did the majority of the code implementation and achieved SoTA performance. Qucheng Gong coded the logic of Bridge game, got initial results of A2C baselines, did ablation studies on the proposed method, designed simple games, provided insights in Contract Bridge, and built automatic UI tools to evaluate trained models against WBridge5. Tina Jiang performed experiments on baselines and ablations on simple games. Correspondence to: Yuandong Tian: \texttt{yuandong@fb.com}.
} \\
Facebook AI Research \\
\And Qucheng Gong \\
Facebook AI Research \\
\And Tina Jiang \\
Facebook AI Research \\
}

\maketitle


%

\newcommand{\fix}{\marginpar{FIX}}
\newcommand{\new}{\marginpar{NEW}}

\newcommand{\s}{\spadesuit}
\newcommand{\h}{\heartsuit}
\newcommand{\dd}{\diamondsuit}
\newcommand{\cc}{\clubsuit}

\def\oursfull{Joint Policy Search}
\def\ours{JPS}
\def\ouragent{\textbf{JPSBid}}
\def\ig{IG}
\def\pg{PG}
\def\pa{\mathrm{pa}}

\begin{abstract}
To learn good joint policies for multi-agent collaboration with imperfect information remains a fundamental challenge. While for two-player zero-sum games, coordinate-ascent approaches (optimizing one agent's policy at a time, e.g., self-play~\cite{silver2018alphazero,heinrich2016deep}) work with guarantees, in multi-agent cooperative setting they often converge to sub-optimal Nash equilibrium. On the other hand, directly modeling joint policy changes in imperfect information game is nontrivial due to complicated interplay of policies (e.g., upstream updates affect downstream state reachability). In this paper, we show global changes of game values can be decomposed to policy changes localized at each information set, with a novel term named \emph{policy-change density}. Based on this, we propose \emph{Joint Policy Search} (JPS) that iteratively improves joint policies of collaborative agents in imperfect information games, without re-evaluating the entire game. On multi-agent collaborative tabular games, JPS is proven to never worsen performance and can improve solutions provided by unilateral approaches (e.g, CFR~\cite{zinkevich2008regret}), outperforming algorithms designed for collaborative policy learning (e.g. BAD~\cite{BAD}). Furthermore, for real-world game with exponential states, \ours{} has an online form that naturally links with gradient updates. We test it to Contract Bridge, a 4-player imperfect-information game where a team of $2$ collaborates to compete against the other. In its bidding phase, players bid in turn to find a good contract through a limited information channel. Based on a strong baseline agent that bids competitive bridge purely through domain-agnostic self-play, JPS improves collaboration of team players and outperforms WBridge5, a championship-winning software, by $+0.63$ IMPs (International Matching Points) per board over $1000$ games, substantially better than previous SoTA ($+0.41$ IMPs/b against WBridge5) under Double-Dummy evaluation. Note that $+0.1$ IMPs/b is regarded as a nontrivial improvement in Computer Bridge. Part of the code is released in \url{https://github.com/facebookresearch/jps}.
\end{abstract}

\section{Introduction}
\vspace{-0.1in}
Deep reinforcement learning has demonstrated strong or even super-human performance in many complex games (e.g., Atari~\cite{dqn-atari}, Dota 2~\cite{openai5}, Starcraft~\cite{vinyals2019grandmaster}, Poker~\cite{Brown418,deepstack}, Find and Seek~\cite{baker2019emergent}, Chess, Go and Shogi~\cite{silver2016alphago, silver2017alphagozero,tian2019elf}). While massive computational resources are used, the underlying approach is quite simple: to iteratively improve current agent policy, assuming stationary environment and fixed policies of all other agents. Although for two-player zero-sum games this is effective, for multi-agent collaborative with imperfect information, it often leads to sub-optimal Nash equilibria where none of the agents is willing to change their policies unilaterally. For example, if speaking one specific language becomes a convention, then unilaterally switching to a different one is not a good choice, even if the other agent actually knows that language better. 

In this case, it is necessary to learn to jointly change policies of multiple agents to achieve better equilibria. One brute-force approach is to change policies of multiple agents simultaneously, and re-evaluate them one by one on the entire game to seek for performance improvement, which is computationally expensive. Alternatively, one might hope that a change of a sparse subset of policies might lead to ``local'' changes of game values and evaluating these local changes can be faster. While this is intuitively reasonable, in imperfect information game (\ig), changing policy on one decision point leads to reachability changes of downstream states, leading to non-local interplay between policy updates.

In this paper, we realize this locality idea by proposing \emph{policy-change density}, a quantity defined at each perfect information history state with two key properties: \textbf{(1)} when summing over all states, it gives overall game value changes upon policy update, and \textbf{(2)} when the local policy remains the same, it vanishes regardless of any policy changes at other parts of the game tree. Based on this density, the value changes of any policy update on a sparse set of decision points can be decomposed into a summation on each decision point (or information set), which is easy and efficient to compute. 

Based on that, we propose a novel approach, called \emph{\oursfull} (\ours). For tabular \ig{}, \ours{} is proven to never worsen the current policy, and is computationally more efficient than brute-force approaches. For simple collaborative  games with enumerable states, we show that \ours{} improves policies returned by Counterfactual Regret Minimization baseline~\cite{zinkevich2008regret} by a fairly good margin, outperforming methods with explicit belief-modeling~\cite{BAD} and Advantageous Actor-Critic (A2C)~\cite{mnih2016asynchronous} with self-play, in particular in more complicated games.

Furthermore, we show \ours{} has a sample-based formulation and can be readily combined with gradient methods and neural networks. This enables us to apply \ours{} to Contract Bridge bidding, in which enumerating the information sets are computationally prohibitive\footnote{\small{In the bidding phase, asides the current player, each of the other 3 players can hold $6.35 \times 10^{11}$ unique hands and there are $10^{47}$ possible bidding sequences. Unlike hint games like Hanabi~\cite{hanabi}, public actions in Bridge (e.g. bid) do not have pre-defined meaning and does not decrease the uncertainty when game progresses.}}. Improved by \ours{} upon a strong A2C baseline, the resulting agent outperforms Wbridge5, a world computer bridge program that won multiple championships, by a large margin of $+0.63$ IMPs per board over a tournament of 1000 games, better than previous state-of-the-art \cite{Gong2019SimpleIB} that beats WBridge5 by $+0.41$ IMPs per board. Note that $+0.1$ IMPs per board is regarded as nontrivial improvement in computer bridge~\cite{baseline19}. 

\section{Related work}
\vspace{-0.1in}
\textbf{Methods to Solve Extensive-Form Games}. For two-player zero-sum extensive-form games, many algorithms have been proposed with theoretical guarantees. For perfect information game (\pg), $\alpha$-$\beta$ pruning, Iterative Deepening depth-first Search~\cite{korf1985depth}, Monte Carlo Tree Search~\cite{coulom2006efficient} are used in Chess~\cite{deepblue} and Go~\cite{silver2016alphago,tian2015better}, yielding strong performances. For imperfect information games (\ig{}), Double-Oracle~\cite{mcmahan2003planning}, Fictitious (self-)play~\cite{heinrich2016deep} and Counterfactual Regret Minimization (CFR~\cite{zinkevich2008regret,lanctot2009monte}) can be proven to achieve Nash equilibrium. These algorithms are \emph{coordinate-ascent}: iteratively find a best response to improve the current policy, given the opponent policies over the history.

On the other hand, it is NP-hard to obtain optimal policies for extensive-form collaborative \ig{} where two agents collaborate to achieve a best common pay-off~\cite{chu2001np}. Such games typically have multiple sub-optimal Nash equilibria, where unilateral policy update cannot help~\cite{foerster2016learning}. Many empirical approaches have been used. Self-play was used in large-scale \ig{} that requires collaboration like Dota 2~\cite{openai5} and Find and Seek~\cite{baker2019emergent}. Impressive empirical performance is achieved with huge computational efforts. Previous works also model belief space (e.g., Point-Based Value Iteration~\cite{Pineau_2006} in POMDP, BAD~\cite{BAD}) or model the behaviors of other agents (e.g., AWESOME~\cite{conitzer2007awesome}, Hyper Q-learning~\cite{tesauro2004extending}, LOLA~\cite{foerster2018learning}). To our best knowledge, we are the first to propose a framework for efficient computation of policy improvement of multi-agent collaborative \ig{}, and show that it can be extended to a sample-based form that is compatible with gradient-based methods and neural networks.

\textbf{Solving Imperfect Information Games}. While substantial progress has been made in \pg{}, how to effectively solve \ig{} in general remains open. Libratus~\cite{Brown418} and Pluribus~\cite{brown2019superhuman} outperform human experts in two-player and multi-player no-limit Texas Holdem with CFR and domain-specific state abstraction, and DeepStack~\cite{deepstack} shows expert-level performance with continual re-solving. ReBeL~\cite{brown2020combining} adapts AlphaZero style self-play to IIG, achieving superhuman level in Poker with much less domain knowledge. Recently,~\cite{lerer2019improving} shows strong performance in Hanabi using collaborative search with a pre-defined common blueprint policy. Suphx~\cite{li2020suphx} achieves superhuman level in Mahjong with supervised learning and policy gradient. DeepRole achieves superhuman level~\cite{serrino2019finding} in \emph{The Resistance: Avalon} with continual re-solving~\cite{deepstack}.

In comparison, Contract Bridge with team collaboration, competition and a huge space of hidden information, remains unsolved. While the playing phase has less uncertainty and champions of computer bridge tournament have demonstrated strong performances against top professionals (e.g., GIB~\cite{ginsberg1999gib}, Jack~\cite{jack}, Wbridge5~\cite{wbridge5}), bidding phase is still challenging due to much less public information. Existing software hard-codes human bidding rules. Recent works~\cite{baseline16, baseline19, Gong2019SimpleIB} use DRL to train a bidding agent, which we compare with. See Sec. 5 for details.

\def\start{\mathrm{start}}
\def\succ{\mathrm{succ}}
\def\down{\mathrm{Down}}
\def\activeSet{\mathrm{Active}}
\def\path{\mathrm{path}}
\def\cI{\mathcal{I}}
\def\noop{\bullet}
\def\cand{\mathrm{cand}}

\begin{figure*}
    \centering
    \includegraphics[width=\textwidth]{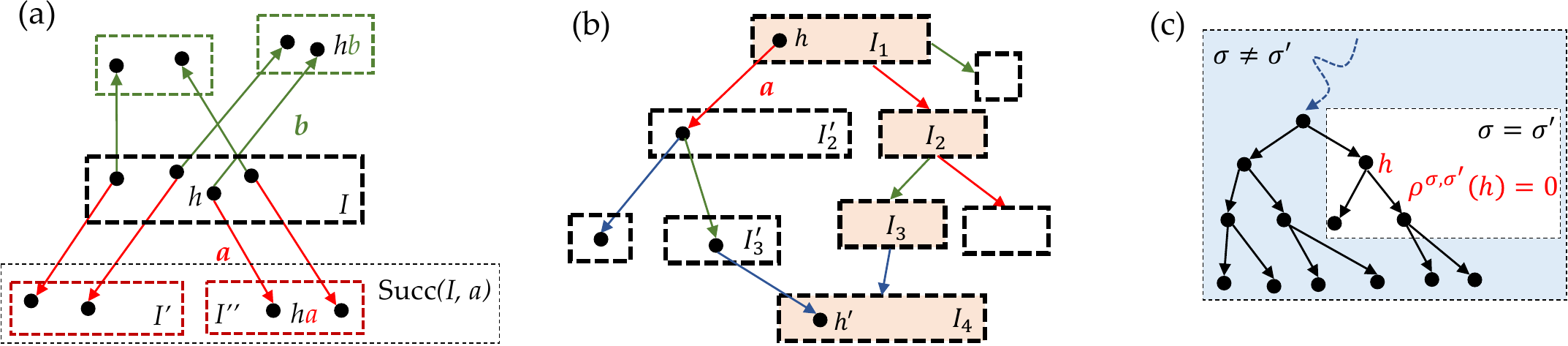}
    \caption{\textbf{(a)} Basic notations. \textbf{(b)} Hard case: a perfect information state $h'$ could first leave active infoset $I_1$, then re-enter the infoset (at $I_4$). Note that it could happen in perfect-recall games, given all the public actions are the same (shown in common red, green and blue edges) and $I_2$ and $I_4$ are played by different players. \textbf{(c)} Our formulation defines \emph{policy-change density} $\vrho^{\sigma,\sigma'}$ that vanishes in regions with $\sigma'=\sigma$, regardless of its upstream/downstream context where $\sigma'\neq\sigma$.}
    \label{fig:info-set}
    \vspace{-0.1in}
\end{figure*}

\section{Background and Notation}
In this section, we formulate our framework in the more general setting of general-sum games, where each of the $C$ players could have a different reward. In this paper, our technique is mainly applied to pure collaborative \ig{}s and we leave its applications in other types of games for future work.

Let $h$ be a perfect information state (or \textbf{state}) of the game. From game start, $h$ is reached via a sequence of public and private actions: $h = a_1a_2\ldots a_d$ (abbreviated as $a_{\le d}$). $I = \{h\}$ is an information set (or \textbf{infoset}) that contains all perfect states indistinguishable from the current player's point of view (e.g., in Poker, $I$ hold all possibilities of opponent cards given public cards and the player's private cards). All $h\in I$ share the same policy $\sigma(h) = \sigma(I)$ and $\sigma(I, a)$ is the probability of taking action $a$. $A(I)$ is the set of allowable actions for infoset $I$.

Let $I(h)$ be the infoset associated with state $h$. $ha$ is the unique next state after taking action $a$ from $h$. $h'$ is a \textbf{descendant} of $h$, denoted as $h\sqsubset h'$, if there exists a sequence of actions $\{a_1,a_2,\ldots, a_d\}$ so that $h' = ha_1a_2\ldots a_d = ha_{\le d}$. The \textbf{successor set} $\succ(I, a)$ contains all the next infosets after taking action $a$ from $I$. The size of $\succ(I, a)$ can be large (e.g., the opponent/partner can make many different decisions based on her private cards). The \textbf{active set} $\cI(\sigma', \sigma) := \{I: \sigma'(I) \neq \sigma(I)\}$ is the collection of infosets where the policy differs between $\sigma$ and $\sigma'$.

$\pi^\sigma(h) := \prod_{i=1}^{d-1} \sigma(I(a_{<i}), a_i)$ is the \textbf{reachability}: the probability of reaching state $h=a_1a_2\ldots a_d$ following the policy $\sigma$. Note that unlike CFR~\cite{zinkevich2008regret}, we use \emph{total} reachability: it includes the probability incurred by chance (or nature) actions and other player's actions under current policy $\sigma$. $Z$ is the \textbf{terminal set}. Each terminal state $z\in Z$ has a reward (or utility) $\vr(z) \in \rr^C$, where $C$ is the number of players. The $i$-th element of $\vr(z)$, $r_i(z)$, is the pay-off of the $i$-th player. 

For state $h\notin Z$, its \textbf{value function} $\vv^\sigma(h) \in \rr^C$ under the current policy $\sigma$ is:
\begin{equation}
    \vv^\sigma(h) = \sum_{a\in A(I(h))} \sigma(I(h), a)\vv^\sigma(ha)
\end{equation}
For terminal node $h\in Z$, its value $\vv^\sigma(z) = \vv(z) = \vr(z)$ is independent of $\sigma$. Intuitively, the value function is the expected reward starting from state $h$ following $\sigma$.  

For \ig{}, what we can observe is infoset $I$ but not state $h$. Therefore we could define \textbf{macroscopic} reachability $\pi^\sigma(I) = \sum_{h\in I}\pi^\sigma(h)$, value function $\vv^\sigma(I)$ and $Q$-function $\vq^\sigma(I, a)$:
\begin{equation}
    \vv^\sigma(I) = \sum_{h\in I} \pi^\sigma(h) \vv^\sigma(h), \quad\quad \vq^\sigma(I, a) = \sum_{h\in I} \pi^\sigma(h) \vv^\sigma(ha)
\end{equation}
and their conditional version: $\vV^\sigma(I) = \vv^\sigma(I)/\pi^\sigma(I)$ and $\vQ^\sigma(I, a) = \vq^\sigma(I, a)/\pi^\sigma(I)$. If we train DRL methods like DQN~\cite{dqn-atari} and A3C~\cite{mnih2016asynchronous} on \ig{} without a discount factor, $\vV^\sigma(I)$ and $\vQ^\sigma(I, a)$ are the terms actually learned in neural networks. As one key difference between \pg{} and \ig{}, $\vv^\sigma(h)$ only depends on the \emph{future} of $\sigma$ after $h$ but $\vV^\sigma(I)$ also depends on the \emph{past} of $\sigma$ before $h$ due to involved reachability: other players' policies affect the reachability of states $h$ within the current infoset $I$, which is invisible to the current player. 

Finally, we define $\bar\vv^\sigma \in \rr^C$ as the overall game value for all $C$ players. $\bar\vv^\sigma := \vv^\sigma(h_0)$ where $h_0$ is the game start (before any chance node, e.g., card dealing). 

\section{A Theoretical Framework for Evaluating Local Policy Change} 
We first start with a novel formulation to evaluate \emph{local} policy change. Local policy means that the active set $\cI(\sigma, \sigma') = \{I : \sigma(I) \neq \sigma'(I)\}$ is small compared to the total number of infosets. A naive approach would be to evaluate the new policy $\sigma'$ over the entire game, which is computationally expensive. 

One might wonder for each policy proposal $\sigma'$, is that possible to decompose $\bar \vv^{\sigma'} - \bar \vv^\sigma$ onto each individual infoset $I\in\cI(\sigma, \sigma')$. However, unlike \pg{}, due to interplay of upstream policies with downstream reachability, a local change of policy affects the utility of its downstream states. For example, a trajectory might leave an active infoset $I_1$ and and later re-enter another active infoset $I_4$ (Fig.~\ref{fig:info-set}(b)). In this case, the policy change at $I_1$ affects the evaluation on $I_4$. Such long-range interactions can be quite complicated to capture. 

This decomposition issue in \ig{} have been addressed in many previous works (e.g., CFR-D~\cite{burch2014solving,burch2018time}, DeepStack~\cite{deepstack}, Reach subgame solving~\cite{brown2017safe}), mainly in the context of solving subgames in a principled way in two-player zero-sum games (like Poker). In contrast, our framework allows simultaneous policy changes at different parts of the game tree, even if they could be far apart, and can work in general-sum games. To our best knowledge, no framework has achieved that so far. 

In this section, we coin a novel quantity called \emph{policy-change density} to achieve this goal.

\def\cO{\mathcal{O}}
\subsection{A Localized Formulation}
\label{sec:formulation}
We propose a novel formulation to \emph{localize} such interactions. For each state $h$, we first define the following \emph{cost} $\vc^{\sigma,\sigma'}\in \rr^C$ and \emph{policy-change density} $\vrho^{\sigma,\sigma'} \in \rr^C$:
\begin{equation}
    \vc^{\sigma,\sigma'}(h) = (\pi^{\sigma'}(h) - \pi^{\sigma}(h))\vv^\sigma(h), \quad\quad\quad 
    \vrho^{\sigma,\sigma'}(h) = -\vc^{\sigma,\sigma'}(h) + \sum_{a\in A(h)} \vc^{\sigma,\sigma'}(ha) \label{eq:c-rho-definition}
\end{equation}
Intuitively, $\vc^{\sigma,\sigma'}(h)$ means if we switch from $\sigma$ to $\sigma'$, what would be the difference in terms of expected reward, if the new policy $\sigma'$ remains the same for all $h$'s descendants. For policy-change density $\vrho^{\sigma,\sigma'}$, the intuition behind its name is clear with the following lemmas:
\begin{lemma}[Density Vanishes if no Local Policy Change]
\label{lemma:density}
For $h$, if $\sigma'(h) = \sigma(h)$, then $\vrho^{\sigma,\sigma'}(h) = \vzero$.
\end{lemma}
\begin{lemma}[Density Summation]
\label{lemma:traj-decomposition}
$\bar \vv^{\sigma'} - \bar \vv^{\sigma} = \sum_{h \notin Z} \vrho^{\sigma,\sigma'}(h)$.
\end{lemma}

Intuitively, Lemma~\ref{lemma:density} shows that $\vrho^{\sigma,\sigma'}$ vanishes if policy does not change \emph{within} a state, regardless of whether policy changes in other part of the game. Lemma~\ref{lemma:traj-decomposition} shows that the summation of density in a subtree can be represented as a function evaluated at its boundary. As a result, \textbf{$\vrho^{\sigma,\sigma'}$ is a \emph{local} quantity with respect to policy change}. In comparison, quantities like $\pi^\sigma$, $v^\sigma$, $c$ and $\pi^{\sigma'}\vv^{\sigma'} - \pi^{\sigma}\vv^{\sigma}$ are \emph{non-local}: e.g., $\vv^\sigma(h)$ (or $\pi^\sigma(h)$) changes if the downstream $\vv^\sigma(h')$ (or upstream $\vv^\sigma(h')$) changes due to $\sigma\rightarrow\sigma'$, even if the local policy remains the same (i.e., $\sigma(h) = \sigma'(h)$). 

With this property, we now address how to decompose $\bar \vv^{\sigma'} - \bar \vv^\sigma$ onto active set $\cI$. According to Lemma~\ref{lemma:density}, for any infoset $I$ with $\sigma'(I)=\sigma(I)$, the policy-change density vanishes. Therefore:
\begin{theorem}[InfoSet Decomposition of Policy Change]
\label{thm:infoset-decomposition}
When $\sigma\rightarrow\sigma'$, the change of game value is:
\begin{equation}
      \bar \vv^{\sigma'} - \bar \vv^{\sigma} = \sum_{I\in \cI} \sum_{h\in I} \vrho^{\sigma,\sigma'}(h) \label{eq:theorem2}
\end{equation}
\end{theorem}
Theorem~\ref{thm:infoset-decomposition} is our main theorem that decomposes local policy changes to each infoset in the active set $\cI$. In the following, we will see how our Joint Policy Search utilizes this property to find a better policy $\sigma'$ from the existing one $\sigma$.

\subsection{Comparison with regret in CFR}
From Eqn.~\ref{eq:c-rho-definition}, we could rewrite the density $\vrho^{\sigma,\sigma'}(h)$ in a more concise form after algebraic manipulation:
\begin{equation}
    \vrho^{\sigma,\sigma'}(h) = \pi^{\sigma'}(h)\left[\sum_{a \in A(I)}\sigma'(I, a)\vv^\sigma(ha) - \vv^\sigma(h)\right]\label{eq:rho-computation}
\end{equation}
Note that this is similar to the regret term in vanilla CFR~\cite{zinkevich2008regret} (Eqn. 7), which takes the form of $\pi^\sigma_{-i}(h)(\vv^\sigma(ha) - \vv^\sigma(h))$ for player $i$ who is to play action $a$ at infoset $I(h)$. In addition that CFR regret uses pure policy $\sigma'$, the key difference here is that we use the total reachability $\pi^{\sigma'}(h)$ evaluated on the \emph{new} policy $\sigma'$, while CFR uses the except-player-$i$ reachability $\pi^{\sigma}_{-i}(h)$ evaluated on the \emph{old} policy $\sigma$.

We emphasize that this small change leads to very different (and novel) theoretical insights. It leads to our policy-change decomposition (Theorem~\ref{thm:infoset-decomposition}) that \emph{exactly} captures the value difference before and after policy changes for general-sum games, while in CFR, summation of the regret at each infoset $I$ is \emph{an upper bound} of the Nash exploitability for two-player zero-sum games. Our advantage comes with a price: the regret in CFR only depends on the old policy $\sigma$ and can be computed independently at each infoset, while computing our policy-change density requires a re-computation of the altered reachability due to new policy $\sigma'$ on the upstream infosets, which will be addressed in Sec.~\ref{sec:jps}. From the derivation, we could also see that the assumption of \emph{perfect recall} is needed to ensure that no double counting exists so that the upper bound can hold (Eqn. 15 in~\cite{zinkevich2008regret}), while there is no such requirement in our formula. We will add these comparisons in the next version.  

\section{Joint Policy Search in Pure Collaborative Multi-agent \ig{}s}
In this paper, we focus on pure collaborative games, in which all players share the same value. Therefore, we could replace $\vv^\sigma$ with a scalar $v^\sigma$. We leave more general cases as the future work. 

For pure collaborative settings, we propose \ours, a novel approach to jointly optimize policies of multiple agents at the same time in \ig{}. Our goal is to find a policy improvement $\sigma'$ so that it is guaranteed that the changes of \emph{overall game value} $\bar v^{\sigma'} - \bar v^\sigma$ is always non-negative. 

\label{sec:tabular-algorithm}

\subsection{Joint Policy Search (\ours)}
\label{sec:jps}
Using Theorem~\ref{thm:infoset-decomposition}, we now can evaluate $\bar v^{\sigma'} - \bar v^{\sigma}$ efficiently. The next step is to have many policy proposals and pick the best one. To perform joint policy search, we pick an active set $\cI$, construct combinations of policy changes at $I\in\cI$ and pick the best policy improvement. To compute policy-change density $\rho^{\sigma,\sigma'}$, before the search starts, we first sweep all the states $h$ to get $v^\sigma(h)$ and $\pi^\sigma(h)$, which can be shared across different policy proposals. During search, the only term we need to compute is the altered reachability $\pi^{\sigma'}$, which depends on upstream policy changes. Therefore, we rank $\cI$ from upstream to downstream and perform a depth-first search (Alg.~\ref{alg:tabular}). The search has the complexity of $\cO(|S| + M)$, where $|S|$ is the total number of states and $M$ is the number of policy candidates. This is more efficient than naive brute-force search that requires a complete sweep of all states for each policy candidate ($\cO(|S|M)$).

\begin{algorithm}
\caption{Joint Policy Search (Tabular form)}
\label{alg:tabular}
\begin{algorithmic}[1]
    \Function{JSP-Main}{$\sigma$}
        \For{$i = 1 \ldots T$}
        \State Compute reachability $\pi^\sigma$ and value $v^\sigma$ under $\sigma$. Pick initial infoset $I_1$.
        \State $\sigma \leftarrow \mathrm{JPS}(\sigma, \{I_1\}, 1)$.
        \EndFor
    \EndFunction
    \Function{JPS}{$\sigma,\cI_\cand$, $d$} \hfill \Comment{\textit{$\cI_\cand$: candidate infosets}}
    \If{$d \ge D$} 
    \State \textbf{return} 0.  \hfill \Comment{\textit{Search reaches maximal depth $D$}}
    \EndIf 
    \For{$I\in \cI_\cand$ and $h\in I$} \Comment{\textit{Set altered reachability with $\sigma'$}}
    \State Compute $\pi^{\sigma'}(h)$ by back-tracing $h' \sqsubset h$ until $I(h')$ is active. Otherwise $\pi^{\sigma'}(h) = \pi^{\sigma}(h)$.
    \EndFor
    \For{$I \in \cI_\cand$ and $a \in A(I)$}  \Comment{\textit{Depth-first Search}}
    \State Set $I$ active. Set $\sigma'(I)$ and accordingly Eqn.~\ref{eq:binary-policy}.
    \State Compute $J^{\sigma, \sigma'}(I) = \sum_{h\in I}\rho^{\sigma, \sigma'}(h)$ using Eqn.~\ref{eq:rho-computation}.
    \State Set $r(I, a) = \mathrm{JPS}(\sigma, \succ(I, a), d+1) + J^{\sigma, \sigma'}(I)$ \Comment{\textit{Recursive Call JPS function}}
    \EndFor
    \State \textbf{return} $\max(0, \max_{I,a} r(I,a))$ \hfill\Comment{Also consider if no infoset in $\cI_\cand$ is active.} 
    \EndFunction
\end{algorithmic}
\end{algorithm}

\textbf{Choice of active set $\cI$ and $\sigma'$}. In our experiment, we choose $\cI = [I_1,\ldots,I_D]$ so that $I_{i+1} \in \succ(I_i, a_i)$ with some $a_i$. On the active set $\cI$, any $\sigma'$ works. In tabular case, we use one-hot policy:
\begin{equation}
    \sigma'(I_i, a) = \mathbb{I}[a = a_i] \label{eq:binary-policy}
\end{equation}
In Alg.~\ref{alg:tabular}, we search over different $a_i$ that determines $\sigma'$ as well as different infosets in $\succ(I_i, a_i)$ to achieve the best performance. Dominated by pure strategies, mixed strategies are not considered.
\begin{theorem}[Performance Guarantee for Alg.~\ref{alg:tabular}]
\label{thm:performance-alg}
    $\bar v^{\sigma'} \ge \bar v^{\sigma}$ for $\sigma' = \mathrm{JSP}\mbox{-}\mathrm{Main}(\sigma)$.  
\end{theorem}

\subsection{Online Joint Policy Search (OJPS)}
\label{sec:online-jps}
To compute quantities in Theorem~\ref{thm:infoset-decomposition}, we still need to compute $\pi^\sigma$ and $v^\sigma$ on all states. This makes it hard for real-world scenarios (e.g., Contract Bridge), where an enumeration of all states is computationally infeasible. Therefore, we consider an online sampling version. Define $J^{\sigma, \sigma'}(I) = \sum_{h\in I}\rho^{\sigma, \sigma'}(h)$ and $J$ can be decomposed into two terms $J(I) = J_1(I) + J_2(I)$ ($\lambda$ is a constant):
\begin{eqnarray}
    J_1(I) &=& \sum_{h\in I}(\pi^{\sigma'}(h) - \lambda\pi^\sigma(h)) \left(\sum_{a\in A(I)}\sigma'(I, a) v^\sigma(ha) - v^\sigma(h)\right) \label{eq:J1} \\
    J_2(I) &=& \lambda\sum_{h\in I} \pi^\sigma(h) \left(\sum_{a\in A(I)} \sigma'(I, a)Q^\sigma(I, a) - V^\sigma(I)\right) \label{eq:J2}
\end{eqnarray}
If we sample a trajectory by running the current policy $\sigma$ and pick one perfect information state $h_0$, then $h_0\sim \pi^\sigma(\cdot)$. Then, for $I=I(h)$, using this sample $h_0$, we can compute $\hat J_1(I) = (\pi^{\sigma'}(h|h_0) - \lambda\pi^\sigma(h|h_0)) \left(\sum_{a} \sigma'(I, a) v^\sigma(ha) - v^\sigma(h)\right)$ and $\hat J_2(I) = \lambda\pi^\sigma(h|h_0)\left(\sum_{a}\sigma'(I, a)Q^\sigma(I, a) - V^\sigma(I)\right)$ can be computed via macroscopic quantities (eg., from neural network). Here $\pi^\sigma(h|h_0) := \pi^\sigma(h) / \pi^\sigma(h_0)$ is the (conditional) probability of reaching $h$ starting from $h_0$. Intuitively, $\hat J_1$ accounts for the benefits of taking actions that favors the current state $h$ (e.g., what is the best policy if all cards are public?), and $\hat J_2$ accounts for effects due to other perfect information states that are not yet sampled. The hyper-parameter $\lambda$ controls their relative importance. Therefore, it is possible that we could use a few perfect information states $h$ to improve imperfect information policy via searching over the best sequence of joint policy change. The resulting action sequence representing joint policy change is sent to the replay buffer for neural network training.

\yuandong{The experiment shows that we should not use one-hot policy, but a $\Delta$ altered policy is more reasonable. We should update accordingly.}

\def\nn{\mathbb{N}}
\def\cA{\mathcal{A}}
\def\cS{\mathcal{S}}

\section{Experiments on Simple Collaborative Games}
\label{sec:tabular-exp}
We try \ours{} on multiple simple two-player pure collaborative \ig{} to demonstrate its effectiveness. Except for private card dealing, all actions in these games are public knowledge with perfect recall. Note that \ours{} can be regarded as a booster to improve any solutions from any existing approaches. 

\begin{definition}[Simple Communication Game of length $L$]
\label{def:comm-game}
Consider a game where $s_1\in \{0, \ldots, 2^{L}-1\}$, $a_1 \in \cA_1 = \{0,1\}$, $a_2 \in \cA_2 \in \{0, \ldots, 2^{L}-1\}$. $P1$ sends one binary public signal for $L$ times, then P2 guess P1's private $s_1$. The reward $r = \mathbf{1}[s_1 = a_2]$ (i.e. $1$ if guess right). 
\end{definition}

\begin{definition}[Simple Bidding Game of size $N$]
\label{def:simple-bidding}
P1 and P2 each dealt a private number $s_1,s_2 \sim \mathrm{Uniform}[0,\ldots, N-1]$. $\cA = \{\mathrm{Pass}, 2^0, \ldots, 2^k\}$ is an ordered set. The game alternates between P1 and P2, and P1 bids first. The bidding sequence is strictly increasing. The game ends if either player passes, and $r = 2^k$ if $s_1 + s_2 \ge 2^k$ where $k$ is the latest bid. Otherwise the contract fails and $r = 0$.
\end{definition}

\begin{definition}[2-Suit Mini-Bridge of size $N$]
\label{def:mini-bridge}
P1 and P2 each dealt a private number $s_1,s_2 \sim \mathrm{Uniform}[0,1,\ldots, N]$. $\cA = \{\mathrm{Pass}, 1\heartsuit, 1\spadesuit, 2\heartsuit, ... N\heartsuit, N\spadesuit\}$ is an ordered set. The game progresses as in Def.~\ref{def:simple-bidding}. Except for the first round, the game ends if either player passes. If $k\spadesuit$ is the last bid and $s_1+s_2 \ge N+k$, or if $k\heartsuit$ is the last bid and $s_1+s_2 \le N-k$, then $r = 2^{k-1}$, otherwise the contract fails ($r = -1$). For pass out situation $(\mathrm{Pass}, \mathrm{Pass})$, $r = 0$.
\end{definition}

\begin{table}[t]
    \centering
    \caption{\small{Average reward of multiple tabular games after optimizing policies using various approaches. Both CFR~\cite{zinkevich2008regret} and CFR1k+\ours{} repeats with 1k different seeds. BAD~\cite{BAD} runs 50 times. The trunk policy network of BAD uses 2 Fully Connected layers with 80 hidden units. Actor-Critic run 10 times. The super script $*$ means the method obtains the best known solution in \emph{one} of its trials. We omit all standard deviations of the mean values since they are $\sim 10^{-2}$.}}
    \label{tbl:tabular-result}
    \small
    \setlength{\tabcolsep}{2pt}
    \vspace{0.1in}
    \begin{tabular}{|c||c|c|c|c||c||c|c|c||c|c|c|}
    \hline
        &  \multicolumn{4}{c}{Comm (Def.~\ref{def:comm-game})} & {Mini-Hanabi} & \multicolumn{3}{c}{Simple Bidding (Def.~\ref{def:simple-bidding})} & \multicolumn{3}{c|}{2SuitBridge (Def.~\ref{def:mini-bridge})} \\ 
        &  $L=3$ &$L=5$ & $L=6$ & $L=7$ & ~\cite{BAD} & $N=4$ & $N=8$ & $N=16$ & $N=3$ & $N=4$ & $N=5$ \\
        \hline
        CFR1k~\cite{zinkevich2008regret} & $0.89^*$ &  $0.85$ & $0.85$& $0.85$ & $9.11^*$ & $2.18^*$ & $4.96^*$ & $10.47$  & $1.01^*$ & $1.62^*$ & $2.60$ \\ 
        CFR1k+\ours & $\mathbf{1.00^*}$ &$\mathbf{1.00^*}$ & $\mathbf{1.00^*}$ & $\mathbf{1.00^*}$ & $\mathbf{9.50^*}$ & $2.20^*$ & $\mathbf{5.00^*}$ & $\mathbf{10.56^*}$ & $\mathbf{1.07^*}$ & $\mathbf{1.71^*}$  &  $\mathbf{2.74^*}$                     \\
        A2C~\cite{mnih2016asynchronous} & $0.60^*$ & $0.57$& $0.51$ &  $0.02$ & $8.20^*$ & $2.19$ & $4.79$ & $9.97$  & $0.66$ & $1.03$ & $1.71$                  \\
        BAD~\cite{BAD} & $\mathbf{1.00^*}$ & $0.88$ & $0.50$ & $0.29$  & $9.47^*$ &  $\mathbf{2.23^*}$ & $4.99^*$ & $9.81$  & $0.53$ & $0.98$ & $ 1.31$ \\
        \hline\hline
        \textbf{Best Known} & 1.00 & 1.00 & 1.00 & 1.00& 10 & 2.25 & 5.06 & 10.75 & 1.13 & 1.84 & 2.89 \\
        \#States   & 633 & 34785 &270273 &2129793 & 53 & 241 & 1985&  16129 &4081& 25576& 147421\\
        \#Infosets & 129 &  2049 & 8193 & 32769& 45 & 61 & 249 & 1009 &1021 &5116 & 24571 \\ \hline
    \end{tabular}
    \label{tab:comm_results}
\end{table}

The communication game (Def.~\ref{def:comm-game}) can be perfectly solved to reach a joint reward of $1$ with arbitrary binary encoding of $s_1$. However, there exists many local solutions where P1 and P2 agree on a subset of $s_1$ but have no consensus on the meaning of a new $L$-bit signal. In this case, a unilateral approach cannot establish such a consensus. The other two games are harder. In Simple Bidding (Def.~\ref{def:simple-bidding}), available actions are on the order of $\log(N)$, requiring P1 and P2 to efficiently communicate. The Mini-Bridge (Def.~\ref{def:mini-bridge}) mimics the bidding phase of Contract Bridge: since bids can only increase, both players need to strike a balance between reaching highest possible contract (for highest rewards) and avoiding overbids that lead to negative rewards. In this situation, forming a convention requires a joint policy improvement for both players.

For SimpleBidding ($N=16$), MiniBridge ($N=4,5$), we run Alg.~\ref{alg:tabular} with a search depth $D = 3$. For other games, we use maximal depth, i.e., from the starting infosets to the terminals. Note this does not involve all infosets, since at each depth only one active infoset exists. JPS never worsens the policy so we use its last solution. For A2C and BAD, we take the best model over 100 epoch (each epoch contains 1000 minibatch updates). Both A2C and BAD use a network to learn the policy, while CFR and \ours{} are tabular approaches. To avoid convergence issue, we report CFR performance after purifying CFR's resulting policy. The raw CFR performance before purification is slightly lower. 

As shown in Tbl.~\ref{tbl:tabular-result}, JPS consistently improves existing solutions in multiple games, in particular for complicated \ig{}s (e.g. 2-Suit Mini-Bridge). Please see Appendix C for a good solution found by JPS in 2-suited Bridge game. BAD~\cite{BAD} does well for simple games but lags behind JPS in more complicated \ig{}s. 

We also tried different combinations between JPS and other solvers. Except for Comm (Def.~\ref{def:comm-game}) that JPS always gets 1.0, uniform random+JPS converges to local minima that CFR is immune to, and under-performs CFR1k+JPS. Combining JPS with more CFR iterations (CFR10k) doesn't improve performance. Compared to CFR1k+JPS, BAD+JPS is worse ($10.47$ vs $10.56$ for $N=16$) in Simple Bidding but \emph{better} ($1.12/1.71/2.77$ vs $1.07/1.71/2.74$ for $N=3/4/5$) in 2-Suit Mini-Bridge. Note that this is quite surprising since the original solutions obtained from BAD are not great but JPS can boost them substantially. We leave these interesting interplays between methods for future study. 

\textbf{Correctness of Theorem~\ref{thm:infoset-decomposition} and runtime speed}. Experiments show that the game value difference $\bar v^{\sigma'}- \bar v^\sigma$ from Theorem~\ref{thm:infoset-decomposition} always coincides with naive computation, with much faster speed. We have compared JPS with brute-force search. For example, for each iteration in Simple Bidding (Def.~\ref{def:simple-bidding}), for $N=8$, JPS takes $\sim1$s while brute-force takes $\sim4$s (4x); for $N=16$ and $d=3$, JPS takes $\sim20$s while brute-force takes $\sim260$s (13x). For communication game (Def.~\ref{def:comm-game}), JPS enjoys a speedup of 3x for $L=4$. For 2-Suit Mini-Bridge of $N=4$, it achieves up to 30x. 

\begin{table}[t]
\centering
\caption{Performance on sample-based version of \ours{}. All CFR1k experiments are repeated 1000 times and all BAD experiments are repeated 50 times. Note that we use sample with replacement so it is possible to get multiple identical samples from one infoset.}
\small
\setlength{\tabcolsep}{2pt}
\begin{tabular}{|c||c|c|c|c|c|c|c|c|c|c|}
\hline 
    & Initialization & All states & \multicolumn{8}{c|}{\#Samples per infoset} \\
                      &  & & 1 & 2 & 5 & 8 & 15 & 20 & 25 & 30 \\ \hline\hline
Mini-Hanabi~\cite{BAD} & CFR1k~\cite{zinkevich2008regret}          & 9.50              & \textbf{10.00}          & 9.99           & 9.95           & 9.75           & 9.51  & 9.51 & 9.51 & 9.51 \\
\hline
SimpleBidding ($N=16$)  & CFR1k~\cite{zinkevich2008regret}          & 10.56             & 10.47     & 10.47     & 10.49     & 10.52     & 10.58 & 10.60  & 10.61 & \textbf{10.61}     \\ \hline
SimpleBidding ($N=16$)  & BAD~\cite{BAD}            & 10.47             & 9.91     & 9.95     & 10.22     & 10.34     & 10.50 & 10.55 & \textbf{10.57} & 10.55     \\ \hline
2-suited Bridge ($N=3$) & BAD~\cite{BAD}            & 1.12              & 0.89      & 1.12      & \textbf{1.13}      & 1.13      & 1.12 & 1.12 & 1.12 & 1.12  \\ \hline
2-suited Bridge ($N=4$) & BAD~\cite{BAD}            & 1.71              & 1.23      & 1.63      & 1.71     & \textbf{1.71}      & 1.68 & 1.67 &  1.68 & 1.69      \\ \hline
2-suited Bridge ($N=5$) & BAD~\cite{BAD}            & 2.77              & 2.12      & 2.51      & 2.74      & \textbf{2.79}     & 2.79  & 2.76 & 2.77 & 2.78    \\ \hline
\end{tabular}
\\
\label{tab:sample-based}
\end{table}

\begin{figure}
    \begin{subfigure}{.5\textwidth}
    \includegraphics[width=\textwidth]{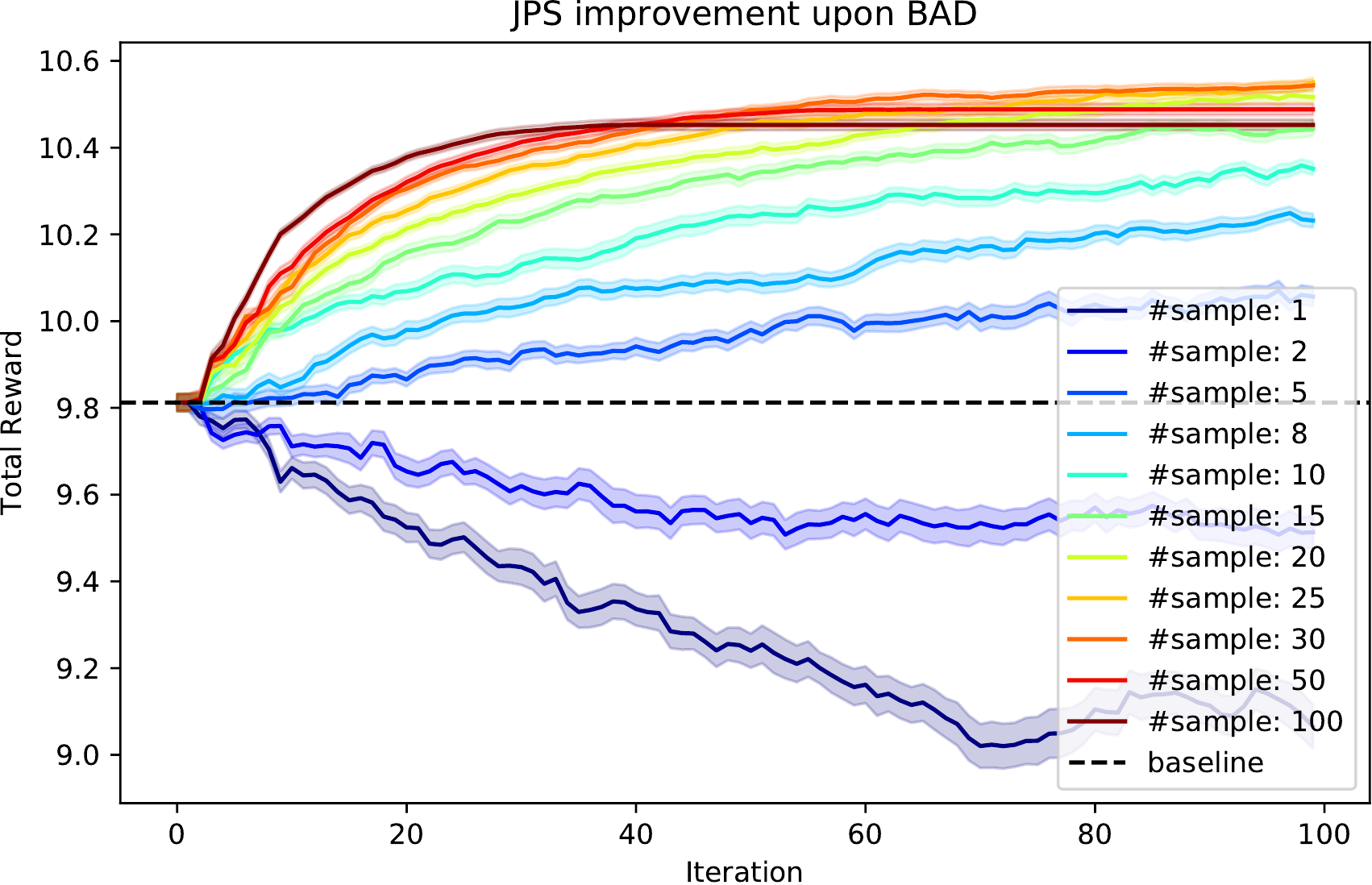}
    \end{subfigure}
    \begin{subfigure}{.5\textwidth}
    \includegraphics[width=\textwidth]{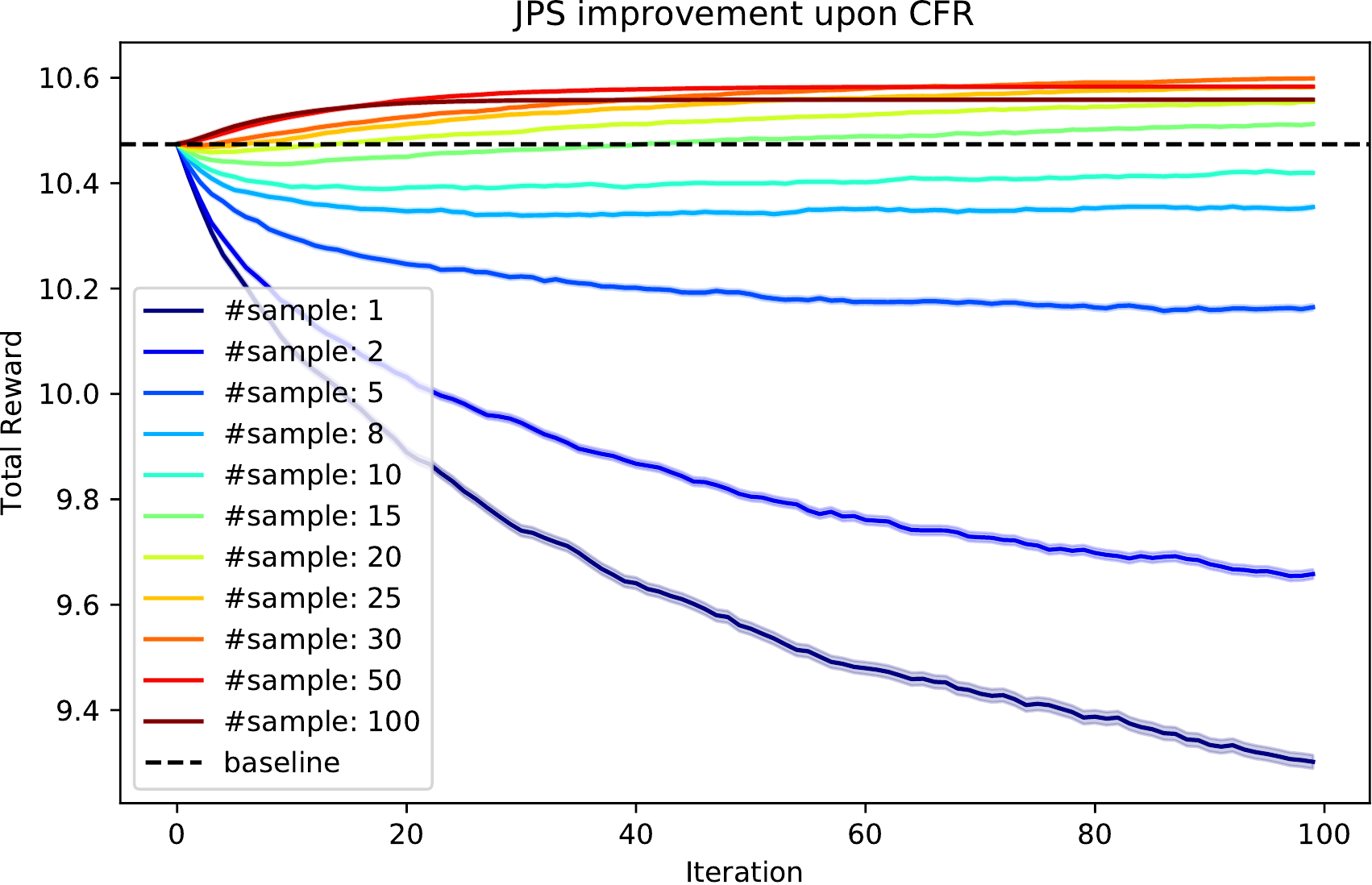}
    \end{subfigure}
    \caption{Improvement of sample-based JPS up on baseline solutions from BAD and CFR1k, using different number of samples per information set, in SimpleBidding ($N=16$). Dashed line is the averaged performance of baselines.}
    \label{fig:sample-based}
\end{figure}

\textbf{Sample-based \ours{}}. Note that Theorem~\ref{thm:infoset-decomposition} sums over all complete states $h\in I$ within each infoset $I$. This could be time-consuming, in particular for real-world games in which one infoset could contain (exponentially) many states. One extension is to sample complete states $h\in I$ and only sum over the samples. In this case, we only need to compute $v^\sigma(h)$ for the sampled state $h$ and their immediate children $v^\sigma(ha)$ using the current policy $\sigma$, saving computational resources. As a trade-off, the monotonity guarantee in Theorem~\ref{thm:performance-alg} breaks and after each iteration of sampled JPS, performance might go down. So we run JPS for 100 iterations and take the best performer (like what we did for BAD). 

Surprisingly, rather than a performance drop, sampled-based \ours{} performs similarly or even \emph{better} than the full-state version, as shown in Tbl.~\ref{tab:sample-based}. In some cases (e.g., Mini-Hanabi), with a single sample, the performance is good, achieving perfect collaboration ($10$), and with more samples the performance degrades towards $9.50$. 
This is due to the fact that sampling could help escape local optima to reach a better solution, while sampling many samples (with replacement) in each infoset would completely cover the infoset and reduce to ``all states'' case, which suffers from local optimum issue.
As a trade-off, for sample-based \ours{}, we lose the performance guarantee. After each iteration the expected return can fluctuate and sometimes collapse into bad solutions (and might jump back later), thus we report the best expected return over iterations. Note that both A2C and BAD are trained with samples (i.e., mini-batches), but their performances still fall behind. 

\section{Application to Contract Bridge Bidding}
In this section, we apply the online version of \ours{} (Sec.~\ref{sec:online-jps}) to the bidding phase of Contract Bridge (a 4-player game, 2 in each team), to improve collaboration between teammates from a strong baseline model. Note that we insert \ours{} in the general self-play framework to improve collaboration between teammates and thus from \ours{}'s point of view, it is still a fully collaborative \ig{} with fixed opponents. Unlike~\cite{baseline16} that only models 2-player collaborative bidding, our baseline and final model are for full Bridge Bidding.

Note that since Bridge is not a pure collaborative games and we apply an online version of JPS, the guarantees of Theorem.~\ref{thm:performance-alg} won't hold. 
 
\textbf{A Crash Course of Bridge Bidding}. The bidding phase of Contract Bridge is like Mini-Bridge (Def.~\ref{def:mini-bridge}) but with a much larger state space (each player now holds a hand with 13 cards from 4 suits). Unlike Mini-Bridge, a player has both her teammate and competitors, making it more than a full-collaborative \ig. Therefore, multiple trade-offs needs to be considered. Human handcrafted conventions to signal private hands, called \textit{bidding systems}. For example, opening bid 2$\h$ used to signal a very strong hand with hearts 
historically, but now signals a weak hand with long hearts. Its current usage blocks opponents from getting their best contract, which happens more frequently than its previous usage (to build a strong heart contract). Please see Appendix A for more details. 

\textbf{Evaluation Metric.} We adopt \emph{duplicate bridge} tournament format: each board (hands of all 4 players) is played twice, where a specific team sits North-South in one game (called open table), and East-West in another (called close table). The final reward is the difference of the results of two tables. This reduces the impact of card dealing randomness and can better evaluate the strength of an agent. 

We use IMPs (International Matching Point) per board, or \imp{}, to measure the strength difference between two Bridge bidding agents. See Appendix A for detailed definition. Intuitively, \imp{} is the normalized score difference between open and close table in duplicate Bridge, ranging from $-24$ to $+24$. In Compute Bridge, a margin of +0.1 \imp{} is considered significant~\cite{baseline19}. In a Bridge tournament, a forfeit in a game counts as $-3$ \imp{}. The difference between a top professional team and an advanced amateur team is about 1.5 \imp{}. 

\textbf{Reward}. We focus on the bidding part of the bridge game and replace the playing phase with Double Dummy Solver (DDS)~\cite{dds}, which computes the maximum tricks each team can get in playing, if all actions are optimal given full information. While this is not how humans plays and in some situations the maximum tricks can only be achieved with full-information, DDS is shown to be a good approximate to human expert plays \cite{baseline19}. Therefore, after bidding we skip the playing phase and directly compute \imp{} from the two tables, each evaluated by DDS, as the only sparse reward.

Note that Commercial software like Wbridge5, however, are not optimized to play under the DDS setting, and we acknowledge that the comparison with Wbridge5 is slightly unfair. We leave end-to-end evaluation including the playing phase as future work. 

\textbf{Dataset}. We generate a training set of 2.5 million hands, drawn from uniform distribution on permutations of 52 cards. We pre-compute their DDS results. The evaluation dataset contains 50k such hands. Both datasets will be open sourced for the community and future work.

\textbf{Baselines}. We use \bsix~\cite{baseline16}, \bnine~\cite{baseline19} and \bs~\cite{Gong2019SimpleIB} as our baselines, all are neural network based methods. See Appendix B for details of each baseline.  

\subsection{Network and Training}
\vspace{-0.1in}
We use the same network architecture as \bs{}, which is also similar to \bnine{}. As show in Fig. \ref{fig:net-and-training-curve}, the network consists of an initial fully connected layer, then 4 fully connected layer with skip connections added every 2 layers to get a latent representation. We use 200 neurons at each hidden layer, so it is much smaller (about 1/70 parameter size of \bnine{}). 
\begin{figure}[ht]
    \begin{subfigure}{.5\textwidth}
        \centering
        \includegraphics[width=0.8\linewidth]{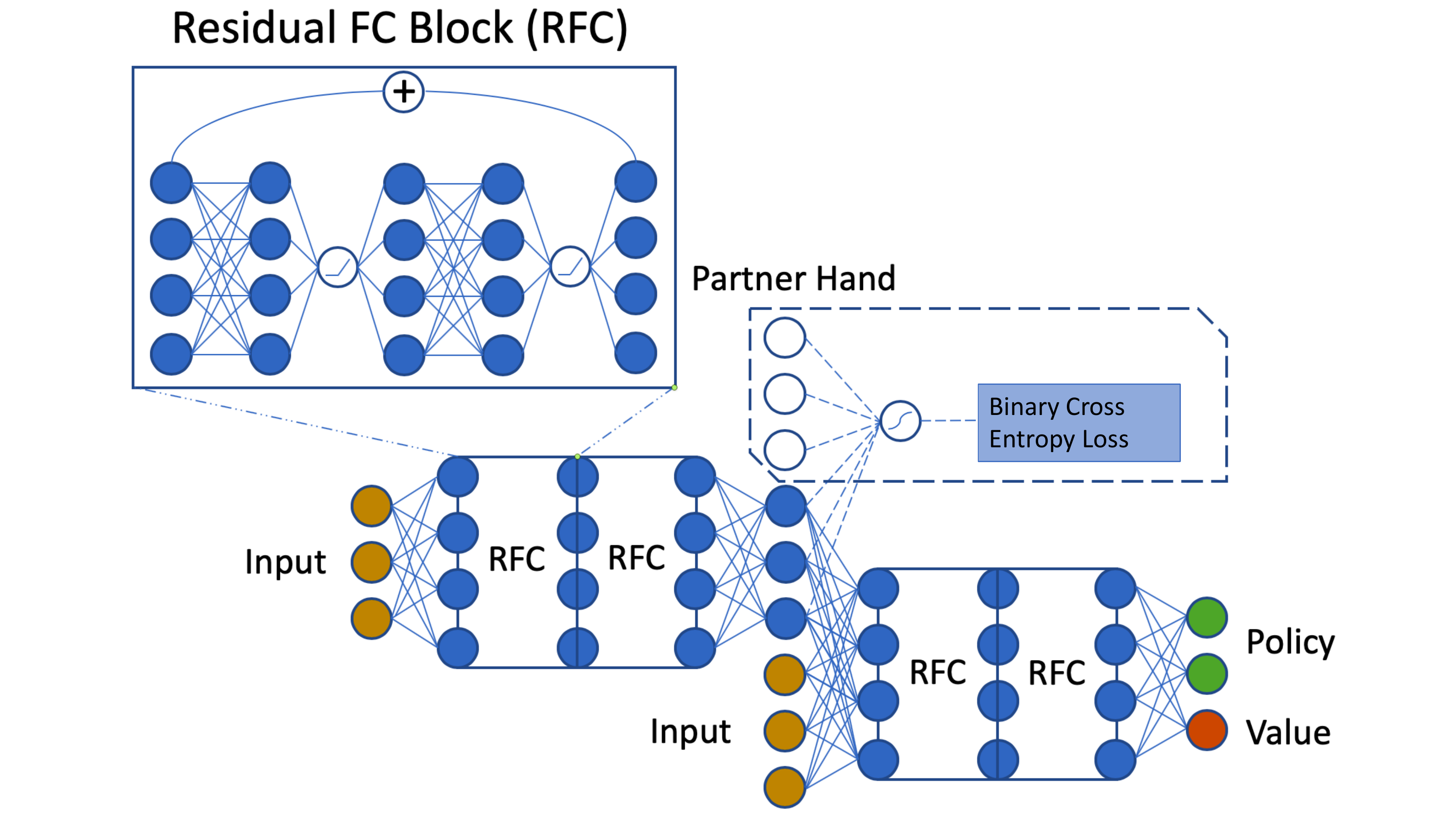}
        \label{fig:net}
    \end{subfigure}
    \begin{subfigure}{.5\textwidth}
        \centering
        \includegraphics[width=0.8\linewidth]{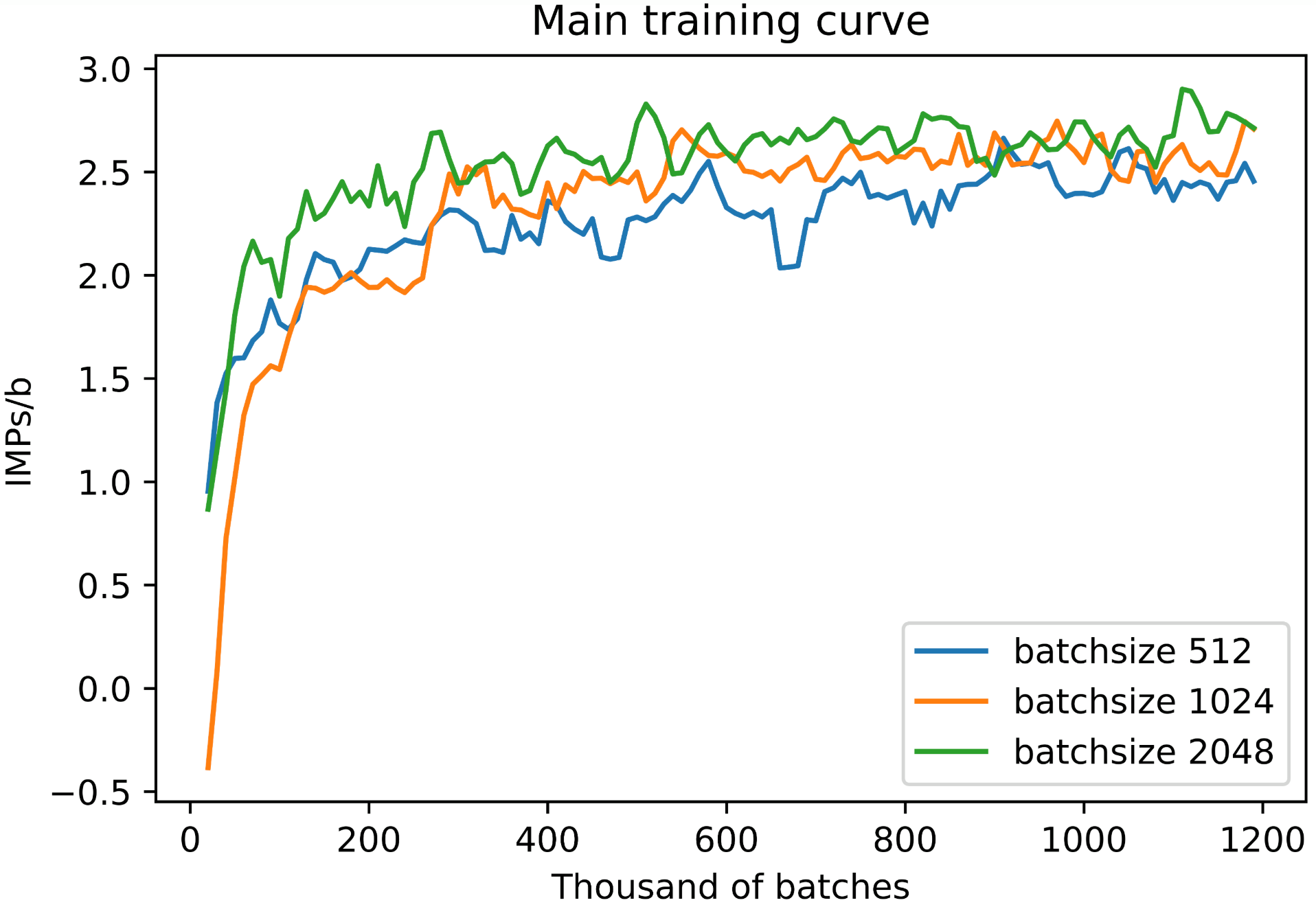}

        \label{fig:training-curve}
    \end{subfigure}
    \vspace{-0.05in}
    \caption{\small{Left: Network Architecture. Supervision from partner's hand is unused in the main results, and is used in the ablation studies.} Right: Smoothed training curves for different batchsizes.
    }
    \label{fig:net-and-training-curve}
\end{figure}

\textbf{Input Representation}. For network input, we use the same encoding as \bs{}. This includes 13 private cards, bidding sequence so far and other signals like vulnerability and legal actions. Please check Appendix~\ref{sec:input-rep} for details. The encoding is general without much domain-specific information. In contrast, \texttt{baseline19} presents a novel bidding history representation using positions in the maximal possible bidding sequence, which is highly specific to Contract Bridge.

\subsection{A Strong Baseline Model} 
We train a strong baseline model for 4-player Bridge Bidding with A2C~\cite{mnih2016asynchronous} with a replay buffer, importance ratio clipping and self-play. During training we run 2000 games in parallel, use batch size of 1024, an entropy ratio of 0.01 and with no discount factor. See Appendix E for details.

Fig.~\ref{fig:net-and-training-curve} shows example training curves against \bsix{}. We significantly outperform \texttt{baseline16} by a huge margin of +2.99 \imp{}. This is partially because \texttt{baseline16} cannot adapt well to competitive bidding setting. Also it can also only handle a fixed length of bids. We have performed an extensive ablation study to find the best combination of common tricks used in DRL. Surprisingly, some of them  believed to be effective in games, e.g., explicit belief modeling, have little impact for Bridge bidding, demonstrating that unilateral improvement of agent's policy is not sufficient.
See Appendix F for a detailed ablation study. 

\def\maxrollout{\texttt{1-search}}
\def\nonsearch{\texttt{non-search}}

\begin{table}[h]
    \centering
   \caption{\small{Fine-tuning RL pre-trained model with search applied on $1\%$ games or moves unless otherwise stated. Performance in \imp{}. 10 baselines are other independently trained actor-critic baselines.}}
    \begin{tabular}{c|c|c}
                         &  vs. baseline & vs. 10 baselines \\ 
        \hline
        \hline
        \nonsearch        &     0.20   & 0.27 $\pm$ 0.13     \\ 
        \maxrollout{}  &   0.46     & 0.37 $\pm$ 0.11 \\
        \hline
        \ours{} (1\%)  & \textbf{0.71}    & 0.47 $\pm$ 0.11 \\
        \ours{} (5\%)  & 0.70     & \textbf{0.66 $\pm$ 0.11} \\
        \ours{} (10\%)  & 0.44    & 0.39 $\pm$ 0.11 \\
    \end{tabular}
    \label{tab:search-vs-no-search}
    \vspace{-0.2in}
\end{table}

\subsection{\ouragent{}: Improving strong baseline models with \ours}
We then use \ours{} to further improve the strong baseline model. Similar to Sec.~\ref{sec:tabular-exp}, \ours{} uses a search depth of $D=3$: the current player's (P1) turn, the opponent's turn and the partner's (P2) turn. We only jointly update the policy of P1 and P2, assuming the opponent plays the current policy $\sigma$. After the P1's turn, we rollout 5 times to sample opponent's actions under $\sigma$. After P2's turn, we rollout 5 times following $\sigma$ to get an estimate of $v^\sigma(h)$. Therefore, for each initial state $h_0$, we run $5\times 5$ rollouts for each combination of policy candidates of P1 and P2. Only a small fraction (e.g., $5\%$) of the games stopped at some game state and run the search procedure above. Other games just follow the current policy $\sigma$ to generate trajectories, which are sent to the replay buffer to stabilize training. A game thread works on one of the two modes decided by rolling a dice.

We also try a baseline \maxrollout{} which only improve P1's policy (i.e., $D=1$). And \nonsearch{} baseline is just to reload the baseline model and continue A2C training. 

From the training, we pick the best model according to its \imp{} against the baseline, and compare with $10$ other baseline models independently trained with A2C with different random seeds. They give comparable performance against \bsix{}. 

Tbl.~\ref{tab:search-vs-no-search} shows a clear difference among \nonsearch{}, \maxrollout{} and \ours{}, in particular in their transfer performance against independent baselines. \ours{} yields much better performance ($+0.66$ \imp{} against 10 independent baselines). We can observe that \maxrollout{} is slightly better than \nonsearch{}. With \ours, the performance gains significantly. 

\textbf{Percentage of search}. Interestingly, performing search in too many games is not only computationally expensive, but also leads to model overfitting, since the trajectories in the replay buffer are infrequently updated. We found that 5\% search performs best against independent baselines. 

\textbf{Against WBridge5}. We train our bot with \ours{} for 14 days and play 1000 games between our bot and WBridge5, a software winning multiple world champion in 2005, 2007, 2008 and 2016. The 1000 games are separately generated, independent of training and evaluation set. We outperform by a margin of $+0.63$ \imp{} with a standard error of $0.22$ \imp{}. This translates to 99.8\% probability of winning in a standard match. This also surpasses the previous SoTAs \bs\cite{Gong2019SimpleIB} ($+0.41$ \imp{} evaluated on 64 games only), and \texttt{baseline19} ($+0.25$ \imp{}). Details in Appendix H. 

Note that we are fully aware of the potential unfairness of comparing with WBridge5 only at Bridge bidding phase. This includes that \textbf{(1)} WBridge5 conforms to human convention but JPS can be creative, \textbf{(2)} WBridge5 optimizes for the results of real Bridge playing rather than double-dummy scores (DDS) that assumes full information during playing, which is obviously very different from how humans play the game. In this paper, to verify our bot, we choose to evaluate against WBridge5, which is an independent baseline tested extensively with both AI and human players. A formal address of these issues requires substantial works and is left for future work.

\textbf{Visualzation of Learned models.} Our learned model is visualized to demonstrate its interesting behaviors (e.g., an aggressive opening table). We leave detailed discussion in the Appendix I. 

\section{Conclusion and Future Work}
In this work, we propose \ours{}, a general optimization technique to jointly optimize policy for collaborative agents in imperfect information game (\ig) efficiently. On simple collaborative games, tabular \ours{} improves existing approaches by a decent margin. Applying online \ours{} in competitive Bridge Bidding yields SoTA agent, outperforming previous works by a large margin ($+0.63$ \imp{}) with a $70\times$ smaller model under Double-Dummy evaluation. As future works, we plan to apply \ours{} to other collaborative \ig{}s, study patterns of sub-optimal equilibria, combine with belief modeling, and use more advanced search techniques.

\section{Broader Impact}
This work has the following potential positive impact in the society:
\begin{itemize}
    \item \ours{} proposes a general formulation and can be applied to multi-agent collaborative games beyond the simple games and Contract Bridge we demonstrate in the paper;
    \item \ours{} can potentially encourage more efficient collaboration between agents and between agents and humans. It might suggest novel coordination patterns, helping jump out of existing (but sub-optimal) social convention.
\end{itemize}
We do not foresee negative societal consequences from \ours{}. 

\bibliography{reference}

\begin{thebibliography}{10}

\bibitem{baker2019emergent}
Bowen Baker, Ingmar Kanitscheider, Todor Markov, Yi~Wu, Glenn Powell, Bob
  McGrew, and Igor Mordatch.
\newblock Emergent tool use from multi-agent autocurricula, 2019.

\bibitem{hanabi}
Nolan Bard, Jakob~N Foerster, Sarath Chandar, Neil Burch, Marc Lanctot,
  H~Francis Song, Emilio Parisotto, Vincent Dumoulin, Subhodeep Moitra, Edward
  Hughes, et~al.
\newblock The hanabi challenge: A new frontier for ai research.
\newblock {\em Artificial Intelligence}, 280:103216, 2020.

\bibitem{brown2020combining}
Noam Brown, Anton Bakhtin, Adam Lerer, and Qucheng Gong.
\newblock Combining deep reinforcement learning and search for
  imperfect-information games.
\newblock {\em arXiv preprint arXiv:2007.13544}, 2020.

\bibitem{brown2017safe}
Noam Brown and Tuomas Sandholm.
\newblock Safe and nested subgame solving for imperfect-information games.
\newblock In {\em Advances in neural information processing systems}, pages
  689--699, 2017.

\bibitem{Brown418}
Noam Brown and Tuomas Sandholm.
\newblock Superhuman ai for heads-up no-limit poker: Libratus beats top
  professionals.
\newblock {\em Science}, 359(6374):418--424, 2018.

\bibitem{brown2019superhuman}
Noam Brown and Tuomas Sandholm.
\newblock Superhuman ai for multiplayer poker.
\newblock {\em Science}, 365(6456):885--890, 2019.

\bibitem{burch2018time}
Neil Burch.
\newblock Time and space: Why imperfect information games are hard.
\newblock 2018.

\bibitem{burch2014solving}
Neil Burch, Michael Johanson, and Michael Bowling.
\newblock Solving imperfect information games using decomposition.
\newblock In {\em Twenty-Eighth AAAI Conference on Artificial Intelligence},
  2014.

\bibitem{deepblue}
Murray Campbell, A~Joseph Hoane~Jr, and Feng-hsiung Hsu.
\newblock Deep blue.
\newblock {\em Artificial intelligence}, 134(1-2):57--83, 2002.

\bibitem{chu2001np}
Francis Chu and Joseph Halpern.
\newblock On the np-completeness of finding an optimal strategy in games with
  common payoffs.
\newblock {\em International Journal of Game Theory}, 30(1):99--106, 2001.

\bibitem{conitzer2007awesome}
Vincent Conitzer and Tuomas Sandholm.
\newblock Awesome: A general multiagent learning algorithm that converges in
  self-play and learns a best response against stationary opponents.
\newblock {\em Machine Learning}, 67(1-2):23--43, 2007.

\bibitem{wbridge5}
Yves Costel.
\newblock Wbridge5.
\newblock \url{http://www.wbridge5.com/}.

\bibitem{coulom2006efficient}
R{\'e}mi Coulom.
\newblock Efficient selectivity and backup operators in monte-carlo tree
  search.
\newblock In {\em International conference on computers and games}, pages
  72--83. Springer, 2006.

\bibitem{foerster2016learning}
Jakob Foerster, Ioannis~Alexandros Assael, Nando De~Freitas, and Shimon
  Whiteson.
\newblock Learning to communicate with deep multi-agent reinforcement learning.
\newblock In {\em Advances in neural information processing systems}, pages
  2137--2145, 2016.

\bibitem{foerster2018learning}
Jakob Foerster, Richard~Y Chen, Maruan Al-Shedivat, Shimon Whiteson, Pieter
  Abbeel, and Igor Mordatch.
\newblock Learning with opponent-learning awareness.
\newblock In {\em Proceedings of the 17th International Conference on
  Autonomous Agents and MultiAgent Systems}, pages 122--130. International
  Foundation for Autonomous Agents and Multiagent Systems, 2018.

\bibitem{BAD}
Jakob~N. Foerster, Francis Song, Edward Hughes, Neil Burch, Iain Dunning,
  Shimon Whiteson, Matthew Botvinick, and Michael Bowling.
\newblock Bayesian action decoder for deep multi-agent reinforcement learning.
\newblock {\em CoRR}, abs/1811.01458, 2018.

\bibitem{ginsberg1999gib}
Matthew~L Ginsberg.
\newblock Gib: Steps toward an expert-level bridge-playing program.
\newblock In {\em IJCAI}, pages 584--593. Citeseer, 1999.

\bibitem{Gong2019SimpleIB}
Qucheng Gong, Yu~Jiang, and Yuandong Tian.
\newblock Simple is better: Training an end-to-end contract bridge bidding
  agent without human knowledge.
\newblock In {\em Real-world Sequential Decision Making Workshop in ICML},
  2019.

\bibitem{dds}
Bo~Haglund.
\newblock Double dummy solver.
\newblock \url{https://github.com/dds-bridge/dds}.

\bibitem{heinrich2016deep}
Johannes Heinrich and David Silver.
\newblock Deep reinforcement learning from self-play in imperfect-information
  games.
\newblock {\em arXiv preprint arXiv:1603.01121}, 2016.

\bibitem{korf1985depth}
Richard~E Korf.
\newblock Depth-first iterative-deepening: An optimal admissible tree search.
\newblock {\em Artificial intelligence}, 27(1):97--109, 1985.

\bibitem{jack}
Hans Kuijf, Wim Heemskerk, and Martin Pattenier.
\newblock Jack bridge.
\newblock \url{http://www.jackbridge.com/eindex.htm}.

\bibitem{lanctot2009monte}
Marc Lanctot, Kevin Waugh, Martin Zinkevich, and Michael Bowling.
\newblock Monte carlo sampling for regret minimization in extensive games.
\newblock In {\em Advances in neural information processing systems}, pages
  1078--1086, 2009.

\bibitem{lerer2019improving}
Adam Lerer, Hengyuan Hu, Jakob Foerster, and Noam Brown.
\newblock Improving policies via search in cooperative partially observable
  games.
\newblock {\em AAAI}, 2020.

\bibitem{li2020suphx}
Junjie Li, Sotetsu Koyamada, Qiwei Ye, Guoqing Liu, Chao Wang, Ruihan Yang,
  Li~Zhao, Tao Qin, Tie-Yan Liu, and Hsiao-Wuen Hon.
\newblock Suphx: Mastering mahjong with deep reinforcement learning, 2020.

\bibitem{mcmahan2003planning}
H~Brendan McMahan, Geoffrey~J Gordon, and Avrim Blum.
\newblock Planning in the presence of cost functions controlled by an
  adversary.
\newblock In {\em Proceedings of the 20th International Conference on Machine
  Learning (ICML-03)}, pages 536--543, 2003.

\bibitem{mnih2016asynchronous}
Volodymyr Mnih, Adria~Puigdomenech Badia, Mehdi Mirza, Alex Graves, Timothy
  Lillicrap, Tim Harley, David Silver, and Koray Kavukcuoglu.
\newblock Asynchronous methods for deep reinforcement learning.
\newblock In {\em International conference on machine learning}, pages
  1928--1937, 2016.

\bibitem{dqn-atari}
Volodymyr Mnih, Koray Kavukcuoglu, David Silver, Alex Graves, Ioannis
  Antonoglou, Daan Wierstra, and Martin~A. Riedmiller.
\newblock Playing atari with deep reinforcement learning.
\newblock {\em CoRR}, abs/1312.5602, 2013.

\bibitem{deepstack}
Matej Morav{\v c}{\'\i}k, Martin Schmid, Neil Burch, Viliam Lis{\'y}, Dustin
  Morrill, Nolan Bard, Trevor Davis, Kevin Waugh, Michael Johanson, and Michael
  Bowling.
\newblock Deepstack: Expert-level artificial intelligence in heads-up no-limit
  poker.
\newblock {\em Science}, 356(6337):508--513, 2017.

\bibitem{openai5}
OpenAI, :, Christopher Berner, Greg Brockman, Brooke Chan, Vicki Cheung,
  Przemysław Dębiak, Christy Dennison, David Farhi, Quirin Fischer, Shariq
  Hashme, Chris Hesse, Rafal Józefowicz, Scott Gray, Catherine Olsson, Jakub
  Pachocki, Michael Petrov, Henrique~Pondé de~Oliveira~Pinto, Jonathan Raiman,
  Tim Salimans, Jeremy Schlatter, Jonas Schneider, Szymon Sidor, Ilya
  Sutskever, Jie Tang, Filip Wolski, and Susan Zhang.
\newblock Dota 2 with large scale deep reinforcement learning, 2019.

\bibitem{Pineau_2006}
J.~Pineau, G.~Gordon, and S.~Thrun.
\newblock Anytime point-based approximations for large pomdps.
\newblock {\em Journal of Artificial Intelligence Research}, 27:335–380, Nov
  2006.

\bibitem{baseline19}
Jiang Rong, Tao Qin, and Bo~An.
\newblock Competitive bridge bidding with deep neural networks.
\newblock {\em CoRR}, abs/1903.00900, 2019.

\bibitem{serrino2019finding}
Jack Serrino, Max Kleiman-Weiner, David~C. Parkes, and Joshua~B. Tenenbaum.
\newblock Finding friend and foe in multi-agent games.
\newblock In {\em NeurIPS}, 2019.

\bibitem{silver2016alphago}
David Silver, Aja Huang, Chris~J Maddison, Arthur Guez, Laurent Sifre, George
  Van Den~Driessche, Julian Schrittwieser, Ioannis Antonoglou, Veda
  Panneershelvam, Marc Lanctot, et~al.
\newblock Mastering the game of go with deep neural networks and tree search.
\newblock {\em nature}, 529(7587):484, 2016.

\bibitem{silver2018alphazero}
David Silver, Thomas Hubert, Julian Schrittwieser, Ioannis Antonoglou, Matthew
  Lai, Arthur Guez, Marc Lanctot, Laurent Sifre, Dharshan Kumaran, Thore
  Graepel, Timothy Lillicrap, Karen Simonyan, and Demis Hassabis.
\newblock A general reinforcement learning algorithm that masters chess, shogi,
  and go through self-play.
\newblock {\em Science}, 362(6419):1140--1144, 2018.

\bibitem{silver2017alphagozero}
David Silver, Julian Schrittwieser, Karen Simonyan, Ioannis Antonoglou, Aja
  Huang, Arthur Guez, Thomas Hubert, Lucas Baker, Matthew Lai, Adrian Bolton,
  et~al.
\newblock Mastering the game of go without human knowledge.
\newblock {\em Nature}, 550(7676):354, 2017.

\bibitem{tesauro2004extending}
Gerald Tesauro.
\newblock Extending q-learning to general adaptive multi-agent systems.
\newblock In {\em Advances in neural information processing systems}, pages
  871--878, 2004.

\bibitem{tian2017elf}
Yuandong Tian, Qucheng Gong, Wenling Shang, Yuxin Wu, and C.~Lawrence Zitnick.
\newblock {ELF:} an extensive, lightweight and flexible research platform for
  real-time strategy games.
\newblock In {\em Advances in Neural Information Processing Systems 30: Annual
  Conference on Neural Information Processing Systems 2017, 4-9 December 2017,
  Long Beach, CA, {USA}}, pages 2656--2666, 2017.

\bibitem{tian2019elf}
Yuandong Tian, Jerry Ma, Qucheng Gong, Shubho Sengupta, Zhuoyuan Chen, James
  Pinkerton, and C~Lawrence Zitnick.
\newblock Elf opengo: An analysis and open reimplementation of alphazero.
\newblock {\em arXiv preprint arXiv:1902.04522}, 2019.

\bibitem{tian2015better}
Yuandong Tian and Yan Zhu.
\newblock Better computer go player with neural network and long-term
  prediction.
\newblock {\em ICLR}, 2016.

\bibitem{tian19}
Zheng Tian, Shihao Zou, Tim Warr, Lisheng Wu, and Jun Wang.
\newblock Learning multi-agent implicit communication through actions: {A} case
  study in contract bridge, a collaborative imperfect-information game.
\newblock {\em CoRR}, abs/1810.04444, 2018.

\bibitem{vinyals2019grandmaster}
Oriol Vinyals, Igor Babuschkin, Wojciech~M Czarnecki, Micha{\"e}l Mathieu,
  Andrew Dudzik, Junyoung Chung, David~H Choi, Richard Powell, Timo Ewalds,
  Petko Georgiev, et~al.
\newblock Grandmaster level in starcraft ii using multi-agent reinforcement
  learning.
\newblock {\em Nature}, 575(7782):350--354, 2019.

\bibitem{baseline16}
Chih{-}Kuan Yeh and Hsuan{-}Tien Lin.
\newblock Automatic bridge bidding using deep reinforcement learning.
\newblock {\em CoRR}, abs/1607.03290, 2016.

\bibitem{zinkevich2008regret}
Martin Zinkevich, Michael Johanson, Michael Bowling, and Carmelo Piccione.
\newblock Regret minimization in games with incomplete information.
\newblock In {\em Advances in neural information processing systems}, pages
  1729--1736, 2008.

\end{thebibliography}
\bibliographystyle{plain}

\clearpage

\appendix
\title{\textit{Supplementary Material for} \\ Joint Policy Search for Multi-agent Collaboration with imperfect Information}

\section{The Contract Bridge Game}
The game of Contract Bridge is played with a standard 52-card deck (4 suits, $\spadesuit$, $\heartsuit$, $\diamondsuit$ and $\clubsuit$, with 13 cards in each suit) and 4 players (North, East, South, West). North-South and East-West are two competitive teams. Each player is dealt with 13 cards. 

There are two phases during the game, namely \textbf{bidding} and \textbf{playing}. After the game, \textbf{scoring} is done based on the won tricks in the playing phase and whether it matches with the contract made in the bidding phase. An example of contract bridge bidding and playing in shown in Fig.~\ref{fig:bid-play}.

\begin{figure}[h]
\centering
\includegraphics[width=0.75\textwidth]{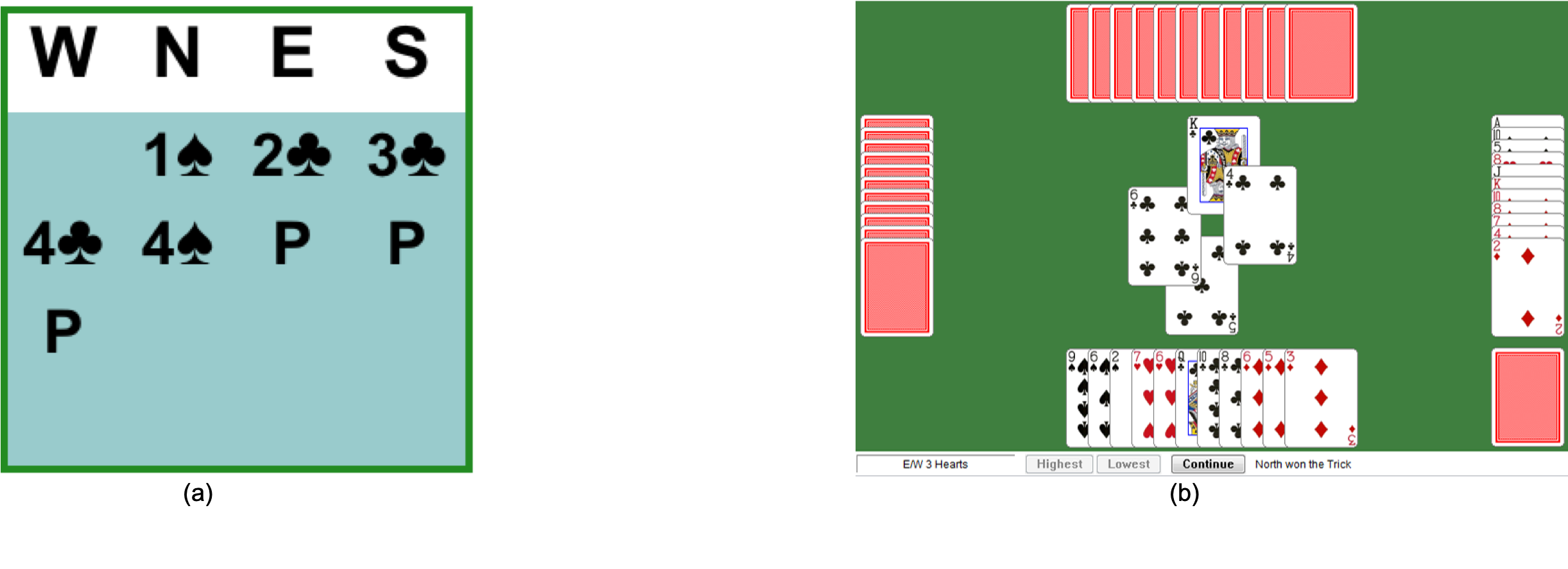}
\caption{\small{\textbf{(a)} A bidding example. North-South prevail and will declare the contract 4$\s$. During the bidding, assuming natural bidding system, the bid 1$\s$, 2$\cc$, 4$\cc$ and 4$\s$ are natural bids, which shows lengths in the nominated suit. The bid 3$\cc$ is an artificial bid, which shows a good hand with $\s$ support for partner, and shows nothing about the $\cc$ suit. To make the contract, North-South needs to take 10 tricks during the playing phase. \textbf{(b)} A playing example. Currently shown is the 2nd round of the playing phase. The dummy's card is visible to all players, and controlled by his partner, declarer. In the current round North player wins with $\cc$K, and will lead the next round.}}
\label{fig:bid-play}
\end{figure}

\textbf{Bidding phase}. During the bidding phase, each player takes turns to bid from 38 available actions. The sequence of bids form an auction. There are 35 contract bids, which consists a level and a strain, ranging from an ordered set \{${1\cc}, {1\dd}, {1\h}, {1\s}, 1NT, {2\cc},.. 7NT$\} where NT stands for No-Trump. The level determines the number of tricks needed to make the contract, and the strain determines the trump suit if the player wins the contract. Each contract bid must be either higher in level or higher in strain than the previous contract bids. 

There are also 3 special bids. \textbf{Pass (P)} is always available when a player is not willing to make a contract bid. Three consecutive passes will terminate the auction, and the last contract bid becomes the final contract, with their side winning the contract. If the auction has 4 Passes, then the game ends with reward $0$ and restarts. \textbf{Double (X)} can be used when either opponent has made a contract bid. It will increase both the contract score, if the declarer makes the contract, and the penalty score for not making the contract. Originally this is used when a player has high confidence that opponent's contract cannot be made, but it can also be used to communicate information. Finally, \textbf{Redouble (XX)} can be used by the declaring team to further amplify the risk and/or reward of a contract, if the contract is doubled. Similarly, this bid can also be used to convey other information.

\textbf{Playing phase}. After the bidding phase is over, the contract is determined, and the owner of the final contract is the declarer. His partner becomes dummy. The other partnership is the defending side. During the playing phase, there are 13 rounds and each rounds the player plays a card. The first round starts with the defending side, and then dummy immediately lays down his cards, and the declarer can control the cards of both himself and dummy. The trump suit is designated by the strain of the final contract (Or None if the strain is NT). Each round, every player has to follow suit. If a player is out of a certain suit, he can play a trump card to beat it. Discarding other suits is always losing in this round. The player who played the highest ranked card (or play a trump) wins a trick, and will play first in the next round. The required number of tricks for the declarer's team to make the contract is contract level + 6 (e.g., $1\spadesuit$ means that 7 tricks are needed). At the end of the game, if the declaring side wins enough tricks, they make the contract. Tricks in addition to the required tricks are called over-tricks. If they fail to make the contract, the tricks short are called under-tricks. 

\textbf{Scoring}.
if the contract is made, the declaring side will receive contract score as a reward, plus small bonuses for over-tricks. Otherwise they will receive negative score determined by under-tricks. Contracts below $4\h$ (except 3NT) are called \textbf{part score contracts}, with relatively low contract scores. Contracts $4\h$ and higher, along with 3NT, are called \textbf{game contracts} with a large bonus score. Finally, Contract with level 6 and 7 are called \textbf{small slams} and \textbf{grand slams} respectively, each with a huge bonus score if made. To introduce more variance, \emph{vulnerability} is randomly assigned to each board to increase bonuses/penalties for failed contracts. 

In \emph{duplicate bridge}, two tables are played with exact the same full-hands. Two players from one team play North-South in one table, and two other players from the same team play East-West in the other table. After raw scores are assigned to each table, the difference is converted to IMPs scale \footnote{\url{https://www.acbl.org/learn_page/how-to-play-bridge/how-to-keep-score/duplicate/}} in a tournament match setting, which is roughly proportional to the square root of the raw score, and ranges from 0 to 24. 

\section{Baselines}
Previous works tried applying DRL on Bridge bidding. 

\bsix~\cite{baseline16} uses DRL to train a bidding model in the collaborative (2-player) setting. It proposes Penetrative Bellman's Equation (PBE) to make the Q-function updates more efficient. The limitation is that PBE can only handle fixed number of bids, which are not realistic in a normal bridge game setting. As suggested by the authors, we modify their pre-trained model to bid competitively (i.e., 4 players), by bidding PASS if the cost of all bids are greater than 0.2. We implement this and further fix its weakness that the model sometimes behaves randomly in a competitive setting if the scenario can never occur in a collaborative setting. We benchmark against them at each episode.  

\bnine~\cite{baseline19} proposes two networks, Estimation Neural Network (ENN) and Policy Neural Network (PNN) to train a competitive bridge model. ENN is first trained supervisedly from human expert data, and PNN is then learned based on ENN. After learning PNN and ENN from human expert data, the two networks are further trained jointly through reinforcement learning and selfplay. PBE claims to be better than Wbridge5 in the collaborative (2-player) setting, while PNN and ENN outperforms Wbridge5 in the competitive (4-player) setting. We could not fully reproduce its results so we cannot directly compare against \bnine{}. However, since both our approach and \bnine{} have compared against WBridge5, we indirectly compare them.

\bs~\cite{Gong2019SimpleIB} is trained with large-scale A2C, similar to our approach, but without the JPS improvement. Furthermore, when evaluating with WBridge5, only 64 games are used. We indirectly compare them thought performance against Wbridge5 on 1000 games.

Policy Belief Learning (PBL)~\cite{tian19} proposes to alternately train between policy learning and belief learning over the whole self-play process. Like~\bsix{}, the Bridge agent obtained from PBL only works in collaborative setting.

\section{Trained Policy on 2-Suit MiniBridge}
We show a learned policy with \ours{} on 2-suit Mini-Bridge with $N=4$ in Tbl. \ref{tab:mini-bridge} , which received the maximal score ($1.84$). We find that the learned policy did well to bid optimal contracts in most scenarios. On the anti-diagonal (0/4 and 4/0 case in the table), no contracts can be made, but in order to explore possible high reward contracts, the agent have to bid, leading to some overbid contracts.

\begin{table}[h]
    \centering 
    \caption{\small{Trained Policy on 2-Suit Mini-Bridge. Rows are the number of $\heartsuit$s of Player 1 and Columns are the number of $\heartsuit$s of Player 2.}}
    \begin{tabular}{c|c|c|c|c|c}
    &0  &  1 & 2 & 3 & 4    \\
    \cmidrule(lr){1-6}
    0 & 1$\spadesuit$-2$\spadesuit$-3$\spadesuit$-4$\spadesuit$-P & 1$\spadesuit$-2$\spadesuit$-3$\spadesuit$-P & 1$\spadesuit$-2$\heartsuit$-2$\spadesuit$-P & 1$\spadesuit$-P& 1$\spadesuit$-2$\heartsuit$-2$\spadesuit$-4$\heartsuit$-4$\spadesuit$-P \\
1 & P-2$\spadesuit$-3$\spadesuit$-P& P-1$\spadesuit$-2$\spadesuit$-P& P-P& P-P& P-1$\heartsuit$-P \\
2 & P-2$\spadesuit$-P& P-1$\spadesuit$-P& P-P& P-P& P-1$\heartsuit$-2$\heartsuit$-P \\
3 & 1$\heartsuit$-1$\spadesuit$-P& 1$\heartsuit$-P& 1$\heartsuit$-P& 1$\heartsuit$-2$\heartsuit$-P& 1$\heartsuit$-3$\heartsuit$-P \\
4 & 1$\heartsuit$-1$\spadesuit$-4$\heartsuit$-4$\spadesuit$-P& 1$\heartsuit$-P& 1$\heartsuit$-P& 1$\heartsuit$-2$\heartsuit$-3$\heartsuit$-P & 1$\heartsuit$-3$\heartsuit$-3$\spadesuit$-4$\heartsuit$-P \\
    \end{tabular}
    \label{tab:mini-bridge}
\end{table}

\begin{figure*}[h]
\centering
\includegraphics[width=1.0\textwidth]{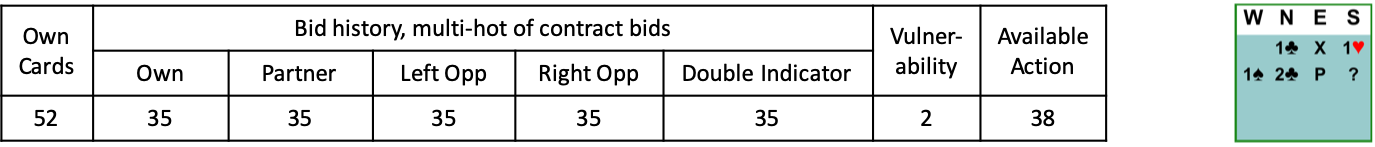}
\caption{Input representation. With the decision point shown in the example, South will mark the following bits in the bidding history encoding: 1$\h$ in "Own" segment, 1$\cc$ and 2$\cc$ in "Partner" segment, $1\s$ in "Left Opp" segment, and $1\cc$ in "Double Indicator" segment.}
\label{fig:input-rep}
\end{figure*} 

\section{Input Representation}
\label{sec:input-rep}
We encode the state of a bridge game to a 267 bit vector as shown in Fig.~\ref{fig:input-rep}. The first 52 bits indicate that if the current player holds a specific card. The next 175 bits encodes the bidding history, which consists of 5 segments of 35 bits each. These 35 bit segments correspond to 35 contract bids. The first segment indicates if the current player has made a corresponding bid in the bidding history. Similarly, the next 3 segments encodes the contract bid history of the current player's partner, left opponent and right opponent. The last segment indicates that if a corresponding contract bid has been doubled or redoubled. Since the bidding sequence can only be non-decreasing, the order of these bids are implicitly conveyed. The next 2 bits encode the current vulnerability of the game, corresponding to the vulnerability of North-South and East-West respectively. Finally, the last 38 bits indicates whether an action is legal, given the current bidding history. 

Note that the representation is \emph{imperfect recall}: from the representation the network only knows some bid is doubled by the opponent team, but doesn't know which opponent doubles that bid. We found that it doesn't make a huge difference in terms of final performance. 

\section{Training Details}
We train the model using Adam with a learning rate of 1e-4. During training we use multinominal exploration to get the action from a policy distribution, and during evaluation we pick the greedy action from the model. We also implement a replay buffer of size 800k, and 80k burn in frames to initialize the replay buffer. 

\textbf{RL Method and Platform Implementation}. We use selfplay on random data to train our baseline models. The baseline model is trained with A2C \cite{mnih2016asynchronous} with replay buffer, off-policy importance ratio correction/capping and self-play, using ReLA platform\footnote{\url{https://github.com/facebookresearch/rela}}. ReLA is an improved version of ELF framework \cite{tian2017elf} using PyTorch C++ interface (i.e., TorchScript). ReLA supports off-policy training with efficient replay buffer. The game logic of Contract Bridge as well as feature extraction steps are implemented in C++ and runs in parallel to make the training fast. Each player can call different models directly in C++ and leaves action trajectories to the common replay buffer, making it suitable for multi-agent setting. The training is thus conducted in a separated Python thread by sampling batches from the replay buffer, and update models accordingly. The updated model is sent back to the C++ side for further self-play, once every \textbf{Sync Frequency} minibatches.

We improve the open source version of ReLA to support dynamic batching in rollouts and search. Unlike ELF that uses thousands of threads for simulation, we now put multiple environments in a single C++ thread to reduce the context-switch cost, while the dynamic batching mechanism can still batch over these environments, using a mechanism provided by \textbf{std::promise} and \textbf{std::future}. This gives $\sim$ 11x speedup compared to a version without batching. The platform is efficient and can evaluate 50k games using pre-trained models in less than a minute on a single GPU. During training, to fill in a replay buffer of 80k transitions, it takes $2.5$ seconds if all agents play with current policy, and $\sim 1$ minute if all agents use \ours{} in 100\% of its actions. The whole training process takes roughly 2 days on 2-GPUs. We also try training a $14$-day version of JPS model.

\section{Ablation Studies}
\subsection{A2C baseline}
We perform extensive ablation studies for A2C self-play models, summarized in Tbl. \ref{tab:perf}. Our attempts to improve its performance by applying existing methods and tuning hyper-parameters yield negative results. 

One example is explicit \textbf{belief modeling} (e.g., with auxiliary loss~\cite{baseline19} or alternating training stages~\cite{tian19}), we found that it doesn't help much in Bridge bidding. We use $L= r L_{belief} + L_{A2C}$ as the loss, where $r$ is a hyper-parameter to control the weight on the auxiliary task, As shown in Table~\ref{tab:perf}, when $r=0$, the model reaches the best performance and the performance decreases as $r$ increase. This shows that it might be hard to move out of local minima with auxiliary loss, compared to search-based approaches. Adding \textbf{more blocks} of FC network cannot further improve its performance, showing that model capacity is not the bottleneck. The performance is similar when the \textbf{sync frequency} is large enough. 

\begin{table}[t]
\centering
\caption{\small{Performance Comparison. The table compares performance when giving different weights to the belief loss and other hyper-parameters such as number of RFC blocks in the network and actor sync frequency.}}

\label{tab:perf}
\begin{tabular}{c|c|c|c|c|c}
    \toprule
    Ratio r    &   imps $\pm$ std & Num Blocks    &   imps $\pm$ std & Sync Frequency     &   imps $\pm$ std  \\
    \cmidrule(lr){1-6}
    0       &   \textbf{2.99 $\pm$ 0.04} &  2       &   2.97 $\pm$ 0.05 & 1       &   2.89 $\pm$ 0.13 \\
    0.001    &   2.86 $\pm$ 0.18 & 4    &   \textbf{2.99 $\pm$ 0.04}&  6      &   2.92 $\pm$ 0.16 \\
    0.01    &   2.77 $\pm$ 0.22 & 10     &   2.94 $\pm$ 0.15 & 12     &   2.94 $\pm$ 0.14 \\
    0.1     &   2.53 $\pm$ 0.27 & 20     &   \textbf{2.99 $\pm$ 0.06} & 50     &  \textbf{2.99 $\pm$ 0.04} \\

    \end{tabular}
    \vspace{-0.1in}
\end{table}

\subsection{Joint Policy Exploration}
It is possible that Joint Policy Search (\ours{}) works just because it encourages joint exploration. To distinguish the two effects, we also run another baseline in which the agent and its partner explore new actions simultaneously but randomly. We find that this hurts the performance, compared to independent exploration. This shows that optimizing the policy of the current player and its partner jointly given the current policy is important for model improvement. 

\begin{table}[h]
    \centering 
    \caption{\small{Joint Exploration hurts the performance.}}
    \begin{tabular}{c|c}
    Joint Random Exploration Ratio  &   imps $\pm$ std    \\
    \cmidrule(lr){1-2}
    0    & \textbf{2.99 $\pm$ 0.04}              \\
    0.001       &   2.43 $\pm$ 0.20    \\
    0.01       &   2.37 $\pm$ 0.31    \\
    \end{tabular}
    \label{tab:search-joint-exploration}
\end{table}

\section{Details of competing with WBridge5 and additional results}
\textbf{Experimental settings.} We compare with WBridge5, which is an award-winning close-sourced free software\footnote{\url{http://www.wbridge5.com/}}. Since it can only run on Microsoft Windows, we implement a UI interface to mimic keyboard and mouse moves to play against WBridge5. Our model controls one player and its partner, while WBridge5 controls the other two players (in summary, 2 \ouragent{} are teamed up against 2 WBridge5 agents). Note that the two players cannot see each other's private information, while their model architecture and parameters are shared. For each model, we use 1000 different board situations and compare its mean estimate (in \imp{}) and standard error of the mean estimate. These 1000 board situations are generated as a separate test set from the training and validation set.

Table~\ref{tab:more-results} shows the performance. Interestingly, while $5\%$ \ouragent{} gives good performance when comparing against 10 independent baselines, it is slightly worse than $1\%$ version when competing with WBridge5. This is likely due to insufficient self-play data produced by expensive rollout operations that involve search. 

\begin{table}[h]
    \centering
    \vspace{-0.1in}
    \caption{\small{Performance against WBridge5.}}
    \begin{tabular}{c|c}
               &  Vs. WBridge5 (\imp{}) \\
    \hline
    A2C baseline  &  $0.29\pm 0.22$ \\
    1\% search, \ouragent{} (2 days) & $0.44 \pm 0.21$ \\
    1\% search, \ouragent{} (14 days) & $\mathbf{0.63 \pm 0.22}$ \\
    5\% search, \ouragent{} (2 days) & $0.38 \pm 0.20$
    \end{tabular}
    \label{tab:more-results}
    \vspace{-0.1in}
\end{table}

\section{Statistics of learned models} 
\subsection{Bidding Statistics}
It is interesting to visualize what the model has learned, and understand some rationales behind the learned conventions. In Fig. \ref{fig:bid_length} and Tbl. \ref{tab:bidding-vis}, we show the bidding length distribution and frequency of each bid used, as well as the distribution of final contracts. We can see that typically agents exchange 6-15 rounds of information to reach the final contract. The agent uses low level bids more frequently and puts an emphasis on $\h$ and $\s$ contracts. The final contract is mostly part scores and game contracts, particularly 3NT, 4$\h$ and 4$\s$. This is because part scores and game contracts are optimal based on DDS for 87\% of hands\footnote{\url{ https://lajollabridge.com/Articles/PartialGameSlamGrand.htm}}. As a result, the model will optimize to reach these contracts.

\begin{figure}
    \centering
    \includegraphics[width=0.4\textwidth]{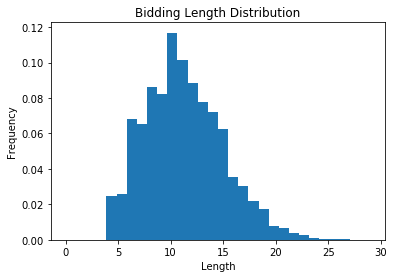}
    \vspace{-0.1in}
    \caption{\small{Bidding length histogram.}}
    \vspace{-0.1in}
    \label{fig:bid_length}
\end{figure}

\begin{table}[h]
    \centering 
        \caption{\small{Most frequent bids and final contracts.}}

    \begin{tabular}{c|c|c|c}
    Bids  &  Frequency & Final Contracts & Frequency   \\
\cmidrule(lr){1-4}
    P   & 57.31\%   & 2$\h$ & 8.07\%     \\
    1$\cc$   & 3.74\%   &2$\s$ & 7.83\%     \\
    1$\s$   & 3.23\%   &1NT & 7.71\%    \\
    X    & 3.16\%   & 3$\dd$ & 7.34\%  \\
    2$\h$   & 3.10\%   &3NT & 6.58\% \\ 
    2$\s$   & 2.84\%   &4$\h$ & 5.90\%  \\
    1NT   & 2.84\%   & 4$\s$ & 5.23\%    \\
    \end{tabular}
    \label{tab:bidding-vis}
\end{table}

\begin{table}[t]
    \centering
    \caption{\small{Opening table comparisons. ``bal'' is abbreviation for a balanced distribution for each suit.}}
    \begin{tabular}{c|c|c}
    Opening bids & Ours & SAYC   \\
    \cmidrule(lr){1-3}
    $1\cc$& 10+ HCP         & 12+ HCP, 3+$\cc$ \\
    $1\dd$&  8-18 HCP, <4 $\h$, <4 $\s$ & 12+ HCP, 3+$\dd$ \\
    $1\h$ &  4-16 HCP, 4-6$\h$ & 12+ HCP, 5+$\h$ \\
    $1\s$ &  4-16 HCP, 4-6$\s$ & 12+ HCP, 5+$\s$ \\
    1NT   & 12-17 HCP, bal   & 15-17 HCP, bal \\
    $2\cc$& 6-13 HCP, 5+$\cc$  & 22+ HCP \\
    $2\dd$& 6-13 HCP, 5+$\dd$  & 5-11 HCP, 6+$\dd$ \\
    $2\h$ & 8-15 HCP, 5+$\h$   & 5-11 HCP, 6+$\h$ \\ 
    $2\s$ & 8-15 HCP, 5+$\s$   & 5-11 HCP, 6+$\s$ \\
    \end{tabular}
    \label{tab:opening} 
\end{table}

\subsection{Opening Table} 
There are two mainstream bidding systems used by human experts. One is called \emph{natural}, where opening and subsequent bids usually shows length in the nominated suit, e.g. the opening bid $1\h$ usually shows 5 or more $\h$ with a decent strength. The other is called \emph{precision}, which heavily relies on relays of bids to partition the state space into meaningful chunks, either in suit lengths or hand strengths, so that the partner knows the distribution of the private card better. For example, an opening bid of $1\cc$ usually shows 16 or more High Card Points (HCP)\footnote{High Card Points is a heuristic to evaluate hand strength, which counts A=4, K=3, Q=2, J=1}, and a subsequent 1$\h$ can show 5 or more $\s$.
To further understand the bidding system the model learns, it is interesting to establish an opening table of the model, defined by the meaning of each opening bid. We select one of the best models, and check the length of each suit and HCP associated with each opening bid. From the opening table, it appears that the model learns a semi-natural bidding system with very aggressive openings (i.e., high bid even with a weak private hand).

\section{Proofs}
\subsection{Lemma~\ref{lemma:density}}
\begin{proof}
Let $I = I(h)$, since $\sigma(I, a) = \sigma'(I, a)$, we have:
\begin{eqnarray}
    \sum_{a\in A(I)} \vc^{\sigma,\sigma'}(ha) &:=& \sum_{a\in A(I)} (\pi^{\sigma'}(ha) - \pi^{\sigma}(ha))\vv^\sigma(ha) \\
    &=& (\pi^{\sigma'}(h) - \pi^{\sigma}(h)) \sum_{a\in A(I)}\sigma(I, a)\vv^\sigma(ha) \\
    &=& (\pi^{\sigma'}(h) - \pi^{\sigma}(h))\vv^\sigma(h) \\
    &=& \vc^{\sigma,\sigma'}(h)
\end{eqnarray}
Therefore, $\vrho^{\sigma,\sigma'}(h) := -\vc^{\sigma,\sigma'}(h) + \sum_{a\in A(I)} \vc^{\sigma,\sigma'}(ha) = \vzero$.
\end{proof}

\subsection{Subtree decomposition}
\begin{lemma}
\label{lemma:general-2}
For a perfect information subtree rooted at $h_0$, we have:
\begin{equation}
    \pi^{\sigma'}(\vv^{\sigma'} - \vv^{\sigma})|_{h_0} = \sum_{h_0\sqsubseteq h\notin Z} \vrho^{\sigma,\sigma'}(h)
\end{equation}
\end{lemma}
\begin{proof}
First by definition, we have for any policy $\sigma'$:
\begin{equation}
\vv^{\sigma'}(h_0) = \sum_{z\in Z} \pi^{\sigma'}(z|h_0)\vv(z) \label{eq:lemma2-1}
\end{equation}
where $\pi^{\sigma'}(z|h_0) := \pi^{\sigma'}(z) / \pi^{\sigma'}(h_0)$ is the \emph{conditional} reachability from $h_0$ to $z$ under policy $\sigma'$. Note that $\vv(z)$ doesn't depend on policy $\sigma'$ since $z$ is a terminal node. 

We now consider each terminal state $z$. Consider a path from game start $h_0$ to $z$: $[h_0, h_1, \ldots, z]$. With telescoping sum, we could write:
\begin{equation}
    \pi^{\sigma'}(z|h_0) \vv(z) = \pi^{\sigma'}(h_0, z|h_0)\vv^\sigma(h_0) + \sum_{h:\ h\sqsubseteq z, ha\sqsubseteq z} \pi^{\sigma'}(ha, z|h_0)\vv^\sigma(ha) - \pi^{\sigma'}(h, z|h_0)\vv^\sigma(h)
\end{equation}
where $\pi^{\sigma'}(h, z|h_0)$ is the joint probability that we reach $z$ through $h$, starting from $h_0$. Now we sum over all possible terminals $z$ that are descendants of $h_0$ (i.e., $h_0 \sqsubseteq z$). Because of the following,
\begin{itemize}
    \item From Eqn.~\ref{eq:lemma2-1}, the left-hand side is $\vv^{\sigma'}(h_0)$;
    \item For the right-hand side, note that $\sum_{z:\ h \sqsubseteq z} \pi^{\sigma'}(h, z|h_0) = \pi^{\sigma'}(h|h_0)$. Intuitively, this means that the reachability of $h$ is the summation of all reachabilities of the terminal nodes $z$ that are the consequence of $h$. 
\end{itemize}
we have:
\begin{equation}
    \vv^{\sigma'}(h_0) = \pi^{\sigma'}(h_0|h_0)\vv^\sigma(h_0) +\sum_{h_0\sqsubseteq h\notin Z} \sum_{a\in A(h)}\pi^{\sigma'}(ha|h_0)\vv^\sigma(ha) - \pi^{\sigma'}(h|h_0)\vv^\sigma(h)
\end{equation}
Notice that $\pi^{\sigma'}(h_0|h_0) = 1$ and if we multiple both side by $\pi^{\sigma'}(h_0)$, we have:
\begin{eqnarray}
\pi^{\sigma'}(\vv^{\sigma'} - \vv^{\sigma})|_{h_0} &=& \sum_{h_0\sqsubseteq h\notin Z} \sum_{a\in A(h)}\pi^{\sigma'}(ha)\vv^\sigma(ha) - \pi^{\sigma'}(h)\vv^\sigma(h) \\
&=& \sum_{h_0\sqsubseteq h\notin Z} \pi^{\sigma'}(h)\sum_{a\in A(h)}\sigma'(I(h), a)\vv^\sigma(ha) - \vv^\sigma(h) \\
&=& \sum_{h_0\sqsubseteq h\notin Z} \vrho^{\sigma,\sigma'}(h)
\end{eqnarray}
This concludes the proof.
\end{proof}

\subsection{Lemma~\ref{lemma:traj-decomposition}}
\begin{proof}
Applying Lemma~\ref{lemma:general-2} and set $h_0$ to be the game start. Then $\pi^{\sigma'}(h_0) = 1$ and all $h$ are descendant of $h_0$ (i.e., $h_0 \sqsubseteq h$):
\begin{equation}
\bar \vv^{\sigma'} - \bar \vv^{\sigma} = \sum_{h\notin Z} \vrho^{\sigma,\sigma'}(h)
\end{equation}
\end{proof}

\subsection{Thm.~\ref{thm:infoset-decomposition}}
\begin{proof}
By Lemma~\ref{lemma:traj-decomposition}, we have:
\begin{equation}
    \bar \vv^{\sigma'} - \bar \vv^{\sigma} = \sum_{h\notin Z} \vrho^{\sigma,\sigma'}(h) = \sum_I\sum_{h\in I} \vrho^{\sigma,\sigma'}(h)
\end{equation}
By Lemma~\ref{lemma:density}, for all infoset set $I$ with $\sigma(I) = \sigma'(I)$, all its perfect information states $h\in I$ has $\vrho^{\sigma,\sigma'}(h) = \vzero$. The conclusion follows.
\end{proof}

\subsection{Thm.~\ref{thm:performance-alg}}
\begin{proof}
    According to Thm.~\ref{thm:infoset-decomposition}, Alg.~\ref{alg:tabular} computes $\bar v^{\sigma'} - \bar v^{\sigma}$ correctly for each policy proposal $\sigma'$ and returns the best $\sigma^*$. Therefore, we have 
    \begin{equation}
        \bar v^{\sigma^*} - \bar v^{\sigma} = \max_{\sigma'} \bar v^{\sigma'} - \bar v^{\sigma} \ge 0
    \end{equation}
\end{proof}

\end{document}